\pdfoutput=1
\documentclass{article} 
\usepackage{preprint,times}

\usepackage{amsmath,amsfonts,bm}

\def\eqref#1{equation~\ref{#1}}
\def\1{\bm{1}}

\DeclareMathAlphabet{\mathsfit}{\encodingdefault}{\sfdefault}{m}{sl}
\SetMathAlphabet{\mathsfit}{bold}{\encodingdefault}{\sfdefault}{bx}{n}

\usepackage{mathrsfs}
\usepackage{amsthm}
\usepackage{hyperref}
\usepackage{url}
\usepackage{enumitem}
\usepackage{booktabs}
\usepackage{multirow}
\usepackage{graphicx}
\usepackage{float}

\title{Detecting Temporal Concepts in Video \\Transformers via a No-False-Positives Test}

\author{Jianghuai Li\thanks{Equal contribution.} \\
University of Chicago \\
\texttt{jianghuai@uchicago.edu}
\and
Andrew Ren\footnotemark[1] \\
University of Chicago \\
\texttt{renqy@uchicago.edu}
\and
Greg Shakhnarovich \\
Toyota Technological Institute at Chicago \\
\texttt{greg@ttic.edu}
}

\newtheorem{theorem}{Theorem}
\theoremstyle{remark}
\newtheorem{remark}{Remark}
\theoremstyle{definition}
\newtheorem{definition}{Definition}

\begin{document}

\maketitle

\begin{abstract}
Video models have the architecture to support encoding temporal aspects of the input, i.e. concepts that require more than one frame to define, like object movement. However, because motion and appearance change are entangled, a feature that appears to encode temporal concepts like object movement may instead encode non-temporal concepts like spatial presence at different positions whose activations merely correlate with motion. To resolve this, we introduce the No-False-Positives (NFP) test, a statistical test with a formal guarantee of statistical specificity for identifying features encoding temporal concepts. Using specially designed stimuli, we apply this test to representations produced by VideoMAE, V-JEPA2, and TimeSformer, and observe that temporal features exist in each model. We validate the test first by showing it yields no temporal features for the DINOv2 image-encoder. Then, using a synthetic representation space, we show that detection rate increases with a direction's projection onto the temporal subspace, and that Sparse Autoencoder features test-flagged as temporal project much more of their norm onto the temporal subspace than unflagged features. Interventions on temporal SAE features yield mixed results across models, suggesting that intervening on temporal features may not outperform a non-temporal feature baseline.
\end{abstract}

\section{Introduction and motivation}
\label{sec:intro}

The success of modern video models on action recognition, temporal ordering and direction identification benchmarks suggests that these models capture temporal information in some form. However, how temporal information is internally organized within these models remains poorly understood. Better understanding could provide improved generalization in tasks requiring temporal reasoning, and open avenues into controllability through model editing and steering. In this paper, we define a feature as a single direction in the representation space, and use these two terms interchangeably. A feature is \emph{temporal} if its content necessarily derives from multiple frames, and \emph{static} if it can be evaluated from a single frame. We are interested in the question: are there temporal features encoded by modern video models in representation space?

This question is difficult to answer because motion and appearance changes are coupled. We call this the \textbf{coupling problem}. For example, an object moving leftwards from position $x$ to $x-\delta$ necessarily produces frames that look different, so a model that appears to understand leftward motion could have features encoding leftward motion (temporal features), or instead a collection of non-temporal concepts ``Object at position $x$'', ``Object at position $x-\delta$'' that together achieve the same series of feature activations and downstream effects. Thus, for any candidate temporal feature, a static-features-only explanation, relying on properties that merely correlate with motion, can plausibly be constructed consistent with observed model behavior. To our knowledge, no prior work has posed this question explicitly, or provided a rigorous method to distinguish between these two possibilities.

We believe the coupling problem is, at its core, a false-positives problem: a test risks misidentifying as temporal a set of static features whose activations or downstream effects happen to match. We need a test with a no-false-positives (NFP) property, so that a feature flagged as temporal is provably not static. We present the first such test (to our knowledge), and deploy it to analyze VideoMAE~\citep{tong2022videomae}, V-JEPA2~\citep{assran2025vjepa2} and TimeSformer~\citep{bertasius2021space}.

There are two key insights enabling our work. First, by using a synthetic set of videos, we can control for what varies across frames, eliminating confounders that would arise in natural video. We use videos of a single white ball (disk) moving on a uniform gray background. The ball's position is the only thing that changes between frames. Thus if a feature's activation varies across frames, that variation must be due to either the ball's current position or its motion. Second, if we maintain a constant distribution of ball positions over the visible screen across every time step $t\in[1, T]$, then any feature whose activation depends only on the ball’s current position, i.e. a static feature, will have the same expected activation at every time step. A value that is constant in expectation across time cannot co-vary with anything, so the expected covariance between any static feature and any temporal concept $\tau$ is exactly zero. By contrapositive, any feature whose activation has nonzero expected covariance with $\tau$ must be non-static. This gives us a setup such that any feature flagged as temporal is provably non-static. We acknowledge there may exist other tests that achieve this no-false-positives property, but none to our knowledge have been developed.

Besides applying our test to analyze video encoders,
we validate the test first by showing it yields no temporal features for the DINOv2 image-encoder. Then, using a synthetic representation space, we show that detection rate increases with a direction's projection onto the temporal subspace, and that Sparse Autoencoder features test-flagged as temporal project their norm onto the planted subspace $198\times$ more than unflagged features. Finally, to attempt a concrete use case, we intervene on the temporal SAE features identified by erasing their span, which yields mixed results across models on SSv2 classification compared to a baseline erasure of static features. This suggests intervening on temporal features does not generally outperform baselines on temporally demanding tasks.

We also present a generalized NFP framework in Appendix~\ref{app:generalNFPframework} beyond the setup of motion features on a moving ball as presented in this paper, and show how the NFP framework can be applied to arbitrary directions as part of an optimization procedure in Appendix~\ref{app:decorrelatestimulus}. 

\section{Existing work on temporal understanding in models}

Prior work has investigated temporal understanding in video models at two levels, neither of which resolves the coupling problem.
Work at the {\bf model level} addresses whether the model as a whole can utilize temporal information present, which is well-studied and broadly affirmed. Work here does not address how temporal information is internally organized. Benchmark evaluations \citep{goyal2017ssv2, resound, onlytimecantell, yuan2024videoglue}, temporal probing tasks \citep{Dave2023NoMS, ChantrySchiappa2022ALR}, and tasks on action direction and ordering \citep{Shao2020FineGymAH, Truong2022DirecFormerAD} all operate at this level.

As an example of the limitations of prior work at this level, \cite{sun2023mme} uses a lack of change in linear probing performance when comparing pretraining with frame shuffled vs original videos as evidence that the VideoMAE model pays little attention to temporal information when performing its reconstruction pre-training task. However, this doesn't clarify if temporal concepts were learned and just not used for their chosen task, or if VideoMAE somehow failed to learn temporal concepts altogether. We show that VideoMAE does encode temporal concepts in representation space. 

At the {\bf representation level}, the question becomes whether temporal concepts like speed are decodable (recoverable with some probe) from a model's whole video representation at a given layer. We see works attempting to separate appearance and motion information \citep{Lee2025, Kowal2022ADD}, probes for geometric and physical properties \citep{Huang2025HowM3, Punzo2026HowDV}, and methods for constructing time-sensitive video representations \citep{Bagad25, Yun2022TimeIM, Ranasinghe2021SelfsupervisedVT, gupta2023siamese}. A study by \cite{joseph2026interpreting} constructed synthetic ball datasets with controlled velocity and acceleration and probed the intermediate representations of V-JEPA2 and VideoMAE-v2. They found layers where motion direction becomes linearly decodable. This is suggestive, but linear decodability does not resolve the coupling problem. A linear probe that recovers speed from the whole video representation at some layer could be reading spatial position patterns like ``Object at position $x$'', ``Object at position $x-\delta$'' that correlate with motion, rather than encodings of speed itself.

In contrast, our work operates at the {\bf feature level}. Prior work at this level examines the use of SAEs in visual models \citep{pach2025sparse,dokme2026interpreting}. We focus on asking whether there are temporal features in video models, and use SAEs as an application use case. An example of another possible question at this level, out of the scope of this paper, would be whether there are non-linear encodings of temporal concepts in representation space. 

\section{Theoretical Framework}
\label{section:theoreticalframework}
 
A \textbf{video} is a sequence of $T$ frames, $V = (f_1, \ldots, f_T)$, where each frame $f_t : \{1, \ldots, H\} \times \{1, \ldots, W\} \to \mathbb{R}^3$ assigns an RGB value to each pixel. We construct synthetic videos of a ball moving on a plain canvas parameterized by:
\begin{itemize}[leftmargin=*, itemsep=1pt, topsep=0pt, parsep=0pt, partopsep=0pt]
    \item Starting position: $(x_0(V), y_0(V)) \in \mathbb{R}^2$
    \item Motion profile: any sequence specifying ball movement, e.g. a velocity profile or a position sequence $(p_0, ...,p_T)$.
\end{itemize}
We use velocity profile as the motion profile for this paper: $v_x(V, \cdot) = (v_x(V, 1), \ldots, v_x(V, T))$ and $v_y(V, \cdot) = (v_y(V, 1), \ldots, v_y(V, T))$. The ball position at frame $t$ is:
\begin{equation}
    x_t(V) = x_0(V) + \sum\nolimits_{j=1}^{t} v_x(V, j), \qquad y_t(V) = y_0(V) + \sum\nolimits_{j=1}^{t} v_y(V, j)
\end{equation}
Since frame $f_t(V)$ is a uniform gray background with a white circle of fixed radius $r$ centered at $(x_t(V), y_t(V))$, the frame content is \textbf{fully determined by the ball position $(x_t, y_t)$}.

\textbf{Tokenization}. 
The video is divided into non-overlapping spatiotemporal patches (tubelets). Each token $j \in \{1, \ldots, J\}$ covers a spatial region of $h \times w$ pixels across $\ell_p$ consecutive frames. We assume $\ell_p=1$ for now, and address the implications of $\ell_p>1$ in Appendix~\ref{app:ellp>1}.

\textbf{Representation and ball-tracking feature activation}.
Consider a video transformer with $L$ layers. After layer $l$, token $j$ has representation $r_j^l(V) \in \mathbb{R}^d$. This is a function of the entire video $V$, because spatiotemporal attention allows information flow across spatial positions and frames. 

A direction in representation space is a vector $w \in \mathbb{R}^d$. Its activation at token $j$ for video $V$ is
$$\phi_w(V, j) = w \cdot r^l_j(V)$$
For a given video $V$ and frame $t$, define $b(V, t)$ as the index $j$ of the token whose spatial region contains the ball's center at frame $t$, then the ball-tracking activation of direction $w$ is $\phi_w(V, b(V,t))$. We then define for any function $F: \mathbb{R}\to\mathbb{R}$:
\begin{equation}
  \psi_w(V, t) = F(\phi_w(V, b(V, t)))  
\end{equation}
The resulting time series $(\psi_w(V, 1), ..., \psi_w(V, T))$ is the series we use to compare against temporal concepts. In Section~\ref{sec:nfp}, we set $F$ as the identity for dimension axes, and ReLU-based functions for SAEs.

\subsection{Definitions}

\begin{definition}[Static Feature]
A direction $w$ is static if there exists a function $g : \mathbb{R}^2 \to \mathbb{R}$ such that:
\begin{equation}
    \phi_w(V, b(V, t)) = g(x_t(V), y_t(V))
\end{equation}
for all videos $V$ in the stimulus set and all $t \in \{1, \ldots, T\}$. 
\end{definition}
$\phi_w(V, b(V,t))$ is in general a function of $V, t$ because of cross-frame attention, so this definition asks whether this can be reduced to a function of just the quantity $(x_t, y_t)$.
This definition makes sense because the frame content of our synthetic videos is fully determined by the ball position $(x_t,y_t)$. As defined, a static feature responds only to the visual content of the current frame, without using information from other frames or the frame index.

This also means that for any $F$, we can define $G=F\circ g:\mathbb{R}^2\to\mathbb{R}$, such that for static directions $w$, $\psi_w(V,t)=F(g(x_t, y_t))=G(x_t,y_t)$. We use $G$ in subsequent sections for generality, and observe that $G = g$ is a special case of $F$ as the identity.

\begin{definition}[Temporal concept profile]
    A temporal concept profile is a sequence $(\tau(V, 1), \ldots, \tau(V, T))$ that describes a per-frame temporal property of video $V$. For example:
\end{definition}
\begin{tabular}{@{}l@{\quad}l@{}}
  Instantaneous speed: & $\tau(V, t) = \sqrt{v_x(V,t)^2 + v_y(V,t)^2}$ \\
  Horizontal velocity: & $\tau(V, t) = v_x(V, t)$ \\
  Acceleration magnitude: & $\tau(V, t) = \|v(V,t{+}1) - v(V,t)\|$ \\
  Direction of motion: & $\tau(V, t) = \mathrm{atan2}(v_y(V,t),\, v_x(V,t))$
\end{tabular}

Since the only varying temporal content in the synthetic videos is that of the ball's motion, $\tau(V, t)$ is determined by the motion profile of $V$. We note that any given temporal concept profile is in general \textbf{not} the motion profile.

\subsection{Stimulus Set}

The stimulus set of synthetic videos $S = \{V_1, \ldots, V_N\}$ is constructed with the following properties.

{\bf (S1) Starting position}. $x_0(V) \sim \text{Uniform}[a_x, b_x]$ and $y_0(V) \sim \text{Uniform}[a_y, b_y]$, drawn independently across videos. The intervals $[a_x, b_x]$ and $[a_y, b_y]$ must satisfy the coverage condition {\bf (S5)}.

{\bf (S2) Motion profiles}. Each video is assigned a velocity profile $(v_x(V, \cdot), v_y(V, \cdot))$ from a distribution of profiles. We show later that to prove the existence of a temporal concept in representation space, a single velocity profile for all videos is sufficient, but to give evidence for a specific direction's association with a specific temporal concept, we need different profiles to allow for decorrelated stimulus (Appendix~\ref{app:decorrelatestimulus}).

{\bf (S3) Independence}. The starting position $(x_0, y_0)$ and the motion profile are mutually independent. Since the temporal concept profile $\tau(V, t)$ is determined by the motion profile, this implies $\tau(V, t)$ is also independent of $(x_0, y_0)$.

{\bf (S4) Off-screen zeroing and compact support}. When the ball center $(x_t,y_t)$ falls outside the visible screen $[0,H]^2$ for our stimulus set, where $H$ is the width and height of the visible screen, there is no ball-containing token, so we set $\psi_w(V,t)=0$ for all directions $w$ as convention. For static features 
this implies compact support, i.e. $G(x,y)=0$ outside $[0,H]^2$. For temporal features this may discard signal, reducing power but not affecting the NFP guarantee.

{\bf (S5) Coverage condition}. For every frame $t$ of every video in $S$, the set of possible ball positions must contain the visible screen $[0, H]^2$. In practice, we compute the cumulative displacements $s_x(V, t) = \sum_{j=1}^t v_x(V, j)$ and $s_y(V, t) = \sum_{j=1}^t v_y(V, j)$, and denote
$$s_x^{\max} = \max_{V \in S,\; t \in \{1, \ldots, T\}} |s_x(V, t)|, \qquad s_y^{\max} = \max_{V \in S,\; t \in \{1, \ldots, T\}} |s_y(V, t)|$$
as the maximum absolute cumulative displacements across all $t$ and all $V\in S$. Then for a visible screen of $[0, H]^2$, the coverage condition requires a ball spawning area $[a_x, b_x]\times[a_y, b_y]$ such that
\begin{equation}
    a_x\leq0-s^{\max}_x,\quad a_y\leq0-s^{\max}_y,\quad b_x\geq H+s^{\max}_x,\quad b_y\geq H+s^{\max}_y
\end{equation}
This means the ball must be able to spawn off-screen, and such videos must be included in the stimulus set, even when the ball never enters the visible frame. 

In short, (S4) and (S5) ensure that the expected activation for static features $\mathbb{E}_{x_0, y_0}[G(x_t,y_t)]$ is the same for every frame $t$, regardless of motion profile; more details can be found in Appendix~\ref{app:compactsupportcoveragecondition}.

\subsection{Test Statistic}

For each video $V$, the within-video covariance between the ball-tracking activation and the temporal concept profile is:
\begin{equation}
    C_w(V) = \text{Cov}_t(\psi_w(V, t),\; \tau(V, t)) = \frac{1}{T}\sum_{t=1}^{T} (\psi_w(V, t) - \bar{\psi}_w(V))(\tau(V, t) - \bar{\tau}(V))
\end{equation}
where $\bar{\psi}_w(V) = \frac{1}{T}\sum_t \psi_w(V, t)$ and $\bar{\tau}(V) = \frac{1}{T}\sum_t \tau(V, t)$.
Note that we use covariance rather than Pearson correlation. In fact, the NFP property is false for Pearson correlation (See Appendix~\ref{app:corrcounterexample}).

\subsection{Proving the NFP property}

\begin{theorem}[No False Positives]
\label{thm:main}
    Let direction $w$ be static: $\psi_w(V, t) = G(x_t(V), y_t(V))$ for some function $G: \mathbb{R}^2 \to \mathbb{R}$ with compact support. Let the stimulus set $S$ satisfy conditions (S1)  - (S5). Then for any temporal concept profile $\tau$ that is independent of $(x_0, y_0)$,
    \begin{equation}
        \mathbb{E}_{x_0, y_0} [C_w(V) \mid \text{Velocity Profile } \pi] = 0
    \end{equation}
    for every fixed velocity profile (motion profile). Consequently, by the law of large numbers:
    \begin{equation}
        \frac{1}{N}\sum_{V \in S} C_w(V) \xrightarrow{N \to \infty} 0
    \end{equation}
\end{theorem}
This means for any one fixed motion profile, as the number of videos $N$ grows, the sample mean covariance goes to zero.
To get the no-false-positives property for stimulus sets containing differing motion profiles (S2), we use the law of total expectation which states that the unconditional expectation (the average across all videos, regardless of profile) is the average of conditional expectations.
\begin{equation}
    \mathbb{E}[C_w(V)] = \mathbb{E}_{\pi} \left[ \mathbb{E}_{x_0, y_0} [C_w(V) \mid \pi] \right],\qquad
    \mathbb{E}[C_w(V)] = \mathbb{E}_{\pi} [0] = 0    
\end{equation}
\textbf{Contrapositive} (the usable form): \textit{If the sample mean $\bar{C}_w = \frac{1}{N}\sum_{V\in S}C_w(V)$ is significantly non-zero (e.g. by a one-sample $t$-test), then $w$ is not static.}

\begin{proof}
Step 1 (Simplification). Expand and simplify the covariance using $\sum_t(\tau(V,t)-\bar{\tau}(V))=0$:
$$C_w(V) = \frac{1}{T}\sum_{t=1}^{T} \psi_w(V, t)(\tau(V, t) - \bar{\tau}(V))$$
Step 2 (Expectation over starting position). Fix the velocity profile (so $\tau(V, t)$ and $\bar{\tau}(V)$ are constants with respect to $x_0, y_0$). By linearity of expectation:
$$\mathbb{E}_{x_0, y_0}[C_w(V)] = \frac{1}{T}\sum_{t=1}^{T} \mathbb{E}_{x_0, y_0}[\psi_w(V, t)] \cdot (\tau(V, t) - \bar{\tau}(V))$$
Step 3 ($t$-independence of expected activation). Since $w$ is static, $\psi_w(V, t) = G(x_t, y_t)$. The position at time $t$ is: $x_t = x_0 + s_x(t), y_t = y_0 + s_y(t)$ where $s_x(t), s_y(t)$ are cumulative displacements (fixed, since we conditioned on the velocity profile). Therefore:
$$\mathbb{E}_{x_0, y_0}[G(x_t, y_t)] = \frac{1}{(b_x - a_x)(b_y - a_y)} \int_{a_x + s_x(t)}^{b_x + s_x(t)} \int_{a_y + s_y(t)}^{b_y + s_y(t)} G(x, y)\, dy\, dx$$
By the coverage condition, the support of $G$ (the screen) lies entirely within the integration region $[a_x + s_x(t), b_x + s_x(t)] \times [a_y + s_y(t), b_y + s_y(t)]$ for every $t$. Since $G(x, y) = 0$ outside the screen, the integral reduces to:
$$\mathbb{E}_{x_0, y_0}[G(x_t, y_t)] = \frac{1}{(b_x - a_x)(b_y - a_y)} \iint_{[0,H]^2} G(x, y)\, dy\, dx = \mu_G$$
We write $\mu_G$ as the mean, \textbf{independent of $t$}, dependent only on $G$ and the spawning range.

Step 4 (Conclusion). Substituting into Step 2:
\[\mathbb{E}_{x_0, y_0}[C_w(V)] = \frac{\mu_G}{T}\sum_{t=1}^{T}(\tau(V, t) - \bar{\tau}(V)) = \mu_G \cdot 0 = 0 \qedhere \] 
\end{proof}

\begin{remark}
    We also see that if $\psi_w(V, t)=\tau(V,t)$ then $\mathbb{E}[C_w]=\mathbb{E}[\text{Var}_t(\tau)]>0$ whenever $\tau$ varies across time. This shows the NFP test does have detection power and can flag temporal $w$.    
\end{remark}

Because the theorem applies to any direction $w$ and any deterministic $F$, the NFP test applies to standard basis axes, PCA, ICA directions, SAE dictionary elements and other choices of function $F$. This also formally justifies the SAE-based interventions of Section~\ref{sec:steer}.

\subsection{Theoretical limitations of the NFP test}
\label{sec:theoreticallimitations}
\textbf{L1. The test detects non-static features, not specific temporal concepts.} For directions $w$ that pass the test, we cannot conclude they encode $\tau$ specifically, because other temporal concepts may covary with $\tau$. A natural follow-up question is whether we are able to recover the true encoding directions of concept $\tau$ if they exist. See Appendix~\ref{app:decorrelatestimulus}, where we discuss this further and provide a mathematical framework. Our analysis suggests that trying to identify a feature as encoding a specific $\tau$ may be futile in the NFP framework. 

\textbf{L2. False negatives are possible}. A temporal feature may have zero covariance with $\tau$ if it encodes a temporal property that is statistically unrelated to $\tau$, or if their relationship is non-linear. The test only detects temporal features that covary with the chosen $\tau$ on the stimulus set. False negatives may also be possible because of insufficient signal in the stimulus set for the particular model, such that covariance does not pass the required threshold.

\textbf{L3. The kind of temporal features detectable depends on the construction of stimulus set S}. We have a motion profile for every video, allowing us to detect ball-motion-related temporal features. We cannot detect other temporal concepts like increasing brightness using this stimulus set. We also consider that multiple confounders, e.g. multiple balls, may be present in a stimulus. We discuss the possibility of handling multiple confounders in Appendix~\ref{app:generalNFPframework}.

\section{Empirical evaluation: General Setup}
\label{sec:setup}

\textbf{Models.} We apply the NFP test to three video transformers that differ in architecture, training, and performance: VideoMAE \citep{tong2022videomae}, a ViT-B masked autoencoder with joint spatiotemporal attention, where we test both the pretrained checkpoint and a variant \emph{finetuned on SSv2}; V-JEPA2 \citep{assran2025vjepa2}, a ViT-L encoder with joint spatiotemporal attention; and TimeSformer \citep{bertasius2021space}, a ViT-B model with divided space-time attention \emph{finetuned on SSv2}. DINOv2 \citep{oquab2024dinov2}, a ViT-B image encoder with no cross-frame attention, serves as a negative control, unable by construction to encode temporal concepts. In all cases, we extract post-MLP residual-stream activations from the final transformer layer. Full model specifications and checkpoint identifiers are in Appendix~\ref{app:model_specs}.

\textbf{NFP test dataset.} We render $N=3000$ videos using Kubric/Blender \citep{greff2022kubric}, each of 16 frames, at either 224 or 256 resolution depending on the model being tested. A white ball of radius $0.9$ m is placed on a uniform grey background under an orthographic camera angle over a $14\times14$ m canvas. Starting positions $(x_0, y_0)$ are drawn uniformly satisfying the coverage condition. We then pick five temporal concepts $\tau_k\in\{\text{speed},\, v_x,\, v_y,\, |a|,\, \text{direction}\}$ and assign each video ball motion profiles using the method outlined in Appendix~\ref{app:decorrelation_practice}. This produces a stimulus set that has all five temporal concepts $\tau_k$ decorrelated pairwise, with covariance $\leq0.0022$ per pair. We discuss why this is important in Appendix~\ref{app:entangled}. For each time step $t$ we record $b(V,t)$ which describes the spatial patch that contains the ball center, and we compute the ground truth values of each of five temporal concept profiles $\tau_k(V,t)$. Computation of $\psi_w$ and $\tau_k(V,t)$ for $\ell_p>1$ is discussed in Appendix~\ref{app:ellp>1}. We provide some visual samples of example videos in Figure~\ref{fig:video_examples}.

\textbf{Temporal Feature Identification.}
We classify features as temporally sensitive based on their within-video covariance with each temporal concept profile. For each pair $(w, k)$ this yields one covariance value per video: $C_w^{(k)}(V) \;=\; \frac{1}{T}\sum_{t=1}^{T} \big(\psi_w(V,t) - \bar\psi_w(V)\big)\big(\tau_k(V,t) - \bar\tau_k(V)\big)$. A one-sample $t$-test then asks whether the mean value for each $(w, k)$ over $N=3000$ videos is zero, which is the population mean if $w$ is static. Direction $w$ is \emph{flagged} as temporal for $\tau_k$ only if $p < \alpha/D$ (Bonferroni) with $\alpha = 0.05$. Features that give undefined $t$-statistics are classed as non-valid features that are neither static nor temporal. 

\textbf{Sparse autoencoder.}
To test the effect of interventions on temporal features and to show that the NFP framework can work in real applications, we train standard L1-penalized ReLU SAEs on the final-layer activations for each model. We follow the procedure in~\citet{pach2025sparse}, and choose L1-penalization instead of TopK methods because TopK SAEs give flickering activations on SAE features \citep{dokme2026interpreting}, where features can be switched on/off rapidly between frames. Given an activation $x$, we go through the encoder: $f = \mathrm{ReLU}(W_\mathrm{enc}(x - b_\mathrm{dec}) + b_\mathrm{enc})$, and the decoder reconstructs $\hat{x} = W_\mathrm{dec} f + b_\mathrm{dec}$. Here, the theoretical framework is applied as $F(\phi_{w_i})=\mathrm{ReLU}(\phi_{w_i}+c)$ for encoder row $w_i$. We use a $\times 8$ expansion factor, and for each video model the SAE is trained on the activations of that model on 4000 randomly chosen SSv2 training clips. For DINOv2, we train an SAE with the same protocol on its per-frame patch token activations. More details on SAE training can be found in Appendix~\ref{app:training}.

\section{Identifying Temporal Features}
\label{sec:nfp}

We apply the NFP test to the models on the dimension axes as directions $w$. Table~\ref{tab:axes_sae_main} gives the flagged counts and the per-profile breakdown. The test flags between 28.7\% and 92.1\% of axes as temporal, varying across models. For all video models except VideoMAE finetuned, the flagged axes counts are concentrated on $|a|$: many more axes are flagged for $|a|$ than for other profiles. This does not mean that these dimension axes are flagged only for $|a|$ or that they encode the concept of $|a|$, only that they encode some (set of) motion concept(s), and their activation covaries significantly with $|a|$. They could in fact also encode static concepts on top of motion concepts via polysemanticity. Likewise, 0 dimension axes flagged as covarying significantly with direction doesn't mean that V-JEPA2 doesn't have some $w$ encoding some concept of direction, rather that it may not match our definition $\tau_{dir}(V, t) = \mathrm{atan2}(v_y(V,t),\, v_x(V,t))$, or that the covariance was deemed insignificant i.e. false negative. In general, these count values are to be taken with an understanding of the theoretical limitations established in Section~\ref{sec:theoreticallimitations}.
\begin{table}[t]
\centering
\small
\begin{tabular}{lllrrrrrrr}
\toprule
Model & $D$ & Features & speed & $v_x$ & $v_y$ & $|a|$ & direction & Any & \% \\
\midrule
\multirow{2}{*}{VideoMAE finetuned} 
    & 768 & axes & 352 & 359 & 151 & 496 & 89 & 707 & 92.1 \\
    & 6144 & SAE & 66 & 39 & 33 & 19 & 4 & 109 & 1.77 \\
\midrule
\multirow{2}{*}{VideoMAE} 
    & 768 & axes & 6 & 23 & 58 & 515 & 29 & 540 & 70.3 \\
    & 6144 & SAE & 0 & 0 & 2 & 170 & 0 & 172 & 2.80 \\
\midrule
\multirow{2}{*}{V-JEPA2} 
    & 1024 & axes & 45 & 2 & 4 & 254 & 0 & 294 & 28.7 \\
    & 8192 & SAE & 4 & 1 & 2 & 57 & 0 & 62 & 0.76 \\
\midrule
\multirow{2}{*}{TimeSformer} 
    & 768 & axes & 23 & 20 & 63 & 408 & 4 & 426 & 55.5 \\
    & 6144 & SAE & 0 & 0 & 0 & 13 & 0 & 13 & 0.21 \\
\midrule
\multirow{2}{*}{DINOv2} 
    & 768 & axes & 0 & 0 & 0 & 0 & 0 & 0 & 0.0 \\
    & 6144 & SAE & 0 & 0 & 0 & 0 & 0 & 0 & 0.0 \\
\bottomrule
\end{tabular}
\caption{Final layer representation space dimension axes flagged by the NFP test at the Bonferroni threshold $p < 0.05/D$. Axes: Dimension axes of the final layer representation. SAE: sparse autoencoder features. $D$ is the number of directions $w$ tested (dimensions or SAE dictionary size). Per-profile columns count features significant for that $\tau_k$, and a feature can be significant for several. ``Any'' counts directions $w$ significant for at least one profile.}
\label{tab:axes_sae_main}
\end{table}

Nonetheless, the much higher counts of axes flagged for $|a|$ seems to suggest something non-trivial in the inner workings of these video models. We do not know why this is the case, and invite further study on the matter. Comparing VideoMAE and VideoMAE finetuned also seems to suggest that finetuning changes the temporal concepts encoded by the model. Since the finetuning was done on SSv2, perhaps learning concepts more directly related to speed, velocity and direction helps with SSv2 classification performance. We can also compare between different layers of the same model (results across different layers in Appendix~\ref{app:main_layers}), which seems to suggest that V-JEPA2 learns temporal information present in videos differently compared to VideoMAE and TimeSformer. For DINOv2, none of the axes are flagged as temporal for any $\tau_k$, which is expected. This is empirical validation that the NFP test doesn't produce false positives on models without temporal-enabling architecture.

Applying the NFP test to an encoder $\to$ SAE setup, we see in Table~\ref{tab:axes_sae_main} that for some models the distribution of the count of flagged SAE features is similar to that of dimension axes. However, for VideoMAE finetuned, there are now many more SAE features flagged as temporal using speed, $v_x$, $v_y$ than for $|a|$, further showing that flagged feature counts should generally be taken with caution. The \% of dictionary elements flagged is also much lower for SAEs than for dimension axes, but the two are not directly comparable because of possible SAE artifacts, such as most SAE features not firing on our stimulus (Appendix~\ref{app:live}). Whether the temporal subspace has meaningful structure, e.g. occupying a dedicated low-dimensional subspace of the representation space, remains an open question that Table~\ref{tab:axes_sae_main} cannot answer. Appendix~\ref{app:decorrelatestimulus} also shows it may be difficult to study such structure using our framework of activations and controlled stimulus, but we believe other methods outside our framework might enable more progress. 

\section{Synthetic validation of the NFP test}
\label{sec:synth_main}

The NFP test on real models tells us which directions are flagged, but we cannot verify against a ground truth because we don't know which directions in video models truly encode temporal concepts. To further validate that the test flags directions carrying temporal content and doesn't flag directions without it, we construct a synthetic representation $h(V,t) \in \mathbb{R}^{768}$ to control what information sits in every direction. We run two validation experiments. The first applies the NFP test directly to directions spanning a controlled range of temporal content. The second trains an SAE on $h$ and checks whether the NFP test's intended behavior is preserved operating on SAE features.
\begin{equation}
h(V,t) = \underbrace{\tau_{\mathrm{norm}}(V,t)\, W_\tau}_{\text{temporal}} + \underbrace{c(V)\, W_\sigma}_{\text{static}} + \underbrace{g(V, \mathrm{pos}(V,t))\, W_p}_{\text{position-dependent}} + \underbrace{\varepsilon(V,t)}_{\text{noise}}
\label{eq:synth}
\end{equation}
$W_\tau \in \mathbb{R}^{5 \times 768}$, $W_\sigma \in \mathbb{R}^{100 \times 768}$, and $W_p \in \mathbb{R}^{50 \times 768}$ have mutually orthonormal rows, obtained from the QR decomposition of a random Gaussian matrix. $\varepsilon(V,t) \sim \mathcal{N}(0, 0.05^2 I)$ is i.i.d. noise.

\textbf{Test on directions.} A test direction $d \in \mathbb{R}^{768}$ is a unit vector constructed as $d = \beta\, u_\tau + \sqrt{1 - \beta^2}\, u_\perp$ where $u_\tau$ is a random unit vector in $\operatorname{span}(W_\tau)$ and $u_\perp$ is a random unit vector in its orthogonal complement. The temporal fraction $\operatorname{ProjFrac}(d) = \|P_{W_\tau} d\|^2/\|d\|^2 = \beta^2$ is exact by construction. 

We ask: as the amount of temporal content in a direction increases from zero to one, does the NFP test's response increase accordingly? And does it correctly return zero for directions with no temporal content, including directions that contain position-dependent content?
At each of 26 log-spaced $\beta^2$ levels from $10^{-5}$ to $1$, we draw 20 random directions $d$ (random $u_\tau$ and $u_\perp$) and 5 single-concept directions where $u_\tau$ is set to one row of $W_\tau$ rather than a random vector in the span. We also test $\beta^2=0$. We compute $\psi(V,t) = \langle h(V,t), d \rangle$ and the NFP test gives a threshold of $|t| > 4.0$. Table~\ref{tab:projfrac} shows a representative subset and the full results are in Appendix~\ref{app:projfrac_full}.

\begin{table}[t]
\centering
\small
\begin{tabular}{lccrrrrr}
\toprule
& & & \multicolumn{5}{c}{$\max_k |t|$, single-concept directions} \\
\cmidrule(lr){4-8}
$\beta^2$ & Flagged & Median $\max_k |t|$ & speed & $v_x$ & $v_y$ & $|a|$ & direction \\
\midrule
0 & 0/200 & 1.5 & -- & -- & -- & -- & -- \\
$10^{-5}$ & 2/20 & 3.3 & 7.2 & 6.3 & 4.5 & 2.6 & 3.5 \\
$1.6\times10^{-5}$ & 6/20 & 3.7 & 7.6 & 7.8 & 1.5 & 1.2 & 5.2 \\
$2.5\times10^{-5}$ & 17/20 & 5.3 & 6.7 & 9.7 & 7.0 & 3.8 & 3.9 \\
$6.3\times10^{-5}$ & 20/20 & 7.6 & 12.8 & 11.8 & 9.9 & 5.3 & 8.6 \\
$10^{-3}$ & 20/20 & 20.6 & 31.7 & 22.8 & 26.1 & 11.3 & 18.5 \\
$1$ & 20/20 & 31.3 & 39.5 & 24.0 & 44.1 & 12.7 & 22.0 \\
\bottomrule
\end{tabular}
\caption{NFP test response across ground-truth temporal fractions $\beta^2$. ``Flagged'' counts the random directions passing the Bonferroni threshold ($|t| > 4.0$) at each level. The last five columns give the response of directions with all temporal content in a single concept.}
\label{tab:projfrac}
\end{table}

No direction at $\beta^2 = 0$ is flagged. Flagging begins at $\beta^2 = 10^{-5}$ and reaches 20/20 by $\beta^2 = 6.3 \times 10^{-5}$, meaning for our synthetic representation the test detects directions whose temporal content occupies less than $0.01\%$ of their squared norm, although real representations may behave differently. The five single-concept directions also saturate at different levels: speed and $v_y$ reach $|t| \approx 40$, while acceleration saturates at $|t| \approx 13$, with acceleration consistently having lowest $t$-scores at each $\beta^2$. In Appendix~\ref{app:saturation} we show that for this synthetic representation the saturation value, which we consider a reasonable proxy for the relative sensitivity of the NFP test to $\tau_k$, is determined by the ratio of within-video to across-video variance of $\tau_k$ in the stimulus set. Intuitively this insight should carry over to real representations. This suggests that the acceleration-dominated pattern we observe in Table~\ref{tab:axes_sae_main} may not be a test artifact, that the prevalence of acceleration-sensitive directions reflects genuine model behavior.

\textbf{Test on validity through SAEs.} We train an SAE on the synthetic representations $h(V,t)$ the same way as the SAEs were trained for real models in Section~\ref{sec:setup}. We then run the NFP test on the learned SAE features $w_i$ and measure $\mathrm{ProjFrac}(w_i) = \|P_{W_{subspace}} w_i\|^2 / \|w_i\|^2$ for all three subspaces.
The NFP test flags 8 of 6144 features ($0.13\%$) as temporal. Table~\ref{tab:synth_alignment} shows that flagged features place $0.582$ of their norm in the temporal subspace, $198\times$ the unflagged value of $0.003$. On the static and position-dependent subspaces, flagged features fall below the random baseline, and the opposite is true for unflagged features. No feature orthogonal to the planted temporal subspace is flagged, and all 8 flagged features align strongly with some planted concept direction (maximum absolute cosine $0.57$–$0.76$). Thus, the NFP test applied through the SAE continues to behave as intended. The individual feature breakdown is in Appendix~\ref{app:synth_features}. We also repeat this control with non-orthonormal temporal directions in Appendix~\ref{app:synth_oblique}, and find that all flagged features remain indeed temporal.
\begin{table}[t]
\centering
\small
\begin{tabular}{lccc}
\toprule
Subspace (dim) & Flagged (8) & Unflagged (6136) & Random baseline \\
\midrule
temporal $W_\tau$ (5) & \textbf{0.582} & 0.003 & 0.007 \\
static $W_\sigma$ (100) & 0.105 & 0.236 & 0.130 \\
position $W_p$ (50) & 0.051 & 0.106 & 0.065 \\
\bottomrule
\end{tabular}
\caption{Mean fraction of SAE feature norm in each planted subspace, for flagged and unflagged features. The random baseline is $d_{\text{sub}}/768$.}
\label{tab:synth_alignment}
\end{table}

\section{Intervention analysis}
\label{sec:steer}

We intervene on SAE features identified as temporal and evaluate intervention effects on SSv2 classification performance. We only present span erasure here, with more results in Appendix~\ref{app:interventionanalysis}. Span erasure seeks to show that erasing the flagged span destroys a model's ability to perform on temporally demanding tasks. We build the orthogonal projector $P = QQ^\top$ onto the span of flagged feature decoder columns, where $Q$ is an orthonormal basis of those columns, and replace every token's layer activation $x$ by $x - Px$ on SSv2 validation videos. This acts directly on the residual stream, and the SAE is used only to supply the erased directions. We do not encode, decode, or re-add a reconstruction error. We test on 10 videos for each SSv2 class, for a total of $n=1740$.

\begin{table}[h]
\centering
\small
\begin{tabular}{@{}lcccccc@{}}
\toprule
Model & Feature Count & Baseline & Identified & Intermediate & Low Static & Random \\
\midrule
VideoMAE finetuned & 109 & 0.629 & 0.572 & 0.595 & 0.602 & 0.602\\
V-JEPA2            & 62 &  0.681 & 0.578 & 0.558 & 0.540 & 0.562\\
TimeSformer        & 16 & 0.495 & 0.460 & 0.437 & 0.455 & 0.490\\
\bottomrule
\end{tabular}
\caption{Columns show all-class accuracy. Baseline: no features erased. Identified: NFP-identified temporal features erased. Low Static: Drawn from features whose covariance with every temporal concept gives $|t| < 2$ (compared to the temporal threshold of $|t|>4+$), span erased. Intermediate: Valid features (see Appendix~\ref{app:live}) that are not identified or low static. Random: Randomly selected from all non-flagged features including non-valid features, performance averaged across 3 seeds. Random can also be seen as an imperfect proxy for features active on SSv2 but not on the stimulus set. TimeSformer SAE features are drawn from layer 10, not layer 11, because the layer 11 representation is not used for SSv2 classification except for the CLS token.}
\label{tab:erasure_main}
\end{table}

Most notably in Table~\ref{tab:erasure_main}, span erasure behavior differs across models. VideoMAE finetuned is most affected by erasure of temporal features, with erasure of random features having a similar effect as erasure of low static features. V-JEPA2 is most affected by erasure of low static features. TimeSformer is equally affected by erasure of temporal and low static features, but erasure of random features has little effect. There are multiple possible reasons for these differences. They could reflect the nature of features used by each model for SSv2 classification; V-JEPA2 and TimeSformer may indeed rely on features that are orthogonal to motion encoding. Or it could be the result of the NFP test failing to flag temporal features, i.e. false negatives, either because there isn't enough signal in the stimulus set (L2) or because the model uses non-ball-motion based temporal features for SSv2 classification (L3). We also acknowledge that the SAE itself may be a point of failure for intervention results.

To attempt to make some sense of these results, we conduct further experiments detailed in Appendix~\ref{app:interventionanalysis}, which together seem to suggest that we are unable to characterize temporal features in video models as more important for temporally demanding tasks compared to non-temporal features. This is unexpected, but perhaps less so for the reasons previously mentioned, and the fact that static features can also be composed across frames in classification MLPs or heads to represent temporal concepts. DINOv2 models have been able to perform respectably on SSv2 through precisely such a mechanism \citep{carreira2024scaling4d}. Overall, interventions on temporal features may not produce intended results better than a non-temporal feature baseline.

\section{Conclusion}

We propose the No-False-Positives test as a solution to the coupling problem in video encoders: a statistical test with a formal guarantee that any feature flagged as temporal is provably non-static. The test relies on two ideas: using synthetic stimuli to control what varies across frames, and constructing a stimulus distribution under which static features have constant expected activation across time, so that nonzero covariance with a temporal concept implies non-static content.

We apply the test to VideoMAE, V-JEPA2, and TimeSformer and find temporal features in all three models, and none in the DINOv2 image encoder as expected. Validation on a synthetic representation shows that the test's detection rate follows the true temporal content of a direction and that the test's intended behavior is preserved through sparse autoencoders. Intervention results are mixed across video models, suggesting that the relationship between detectable temporal features in representation space and task-relevant temporal reasoning varies across models.

Several improvements could be made. Using different stimulus sets would broaden the kinds of temporal concepts detectable. Importantly, more work should also be done in light of intervention results to understand whether the proposed mechanism of static feature $\to$ composed temporal understanding in classification heads is indeed how many video models work, as well as to investigate why video model pretraining seemingly fails to alone produce impactful temporal features.

\newpage

\section*{Reproducibility}
Details of the NFP test dataset construction and sparse autoencoder usage can be found in Section~\ref{sec:setup}. We provide more information on model specification in Appendix~\ref{app:model_specs} and SAE training in Appendix~\ref{app:training}. We also provide supplementary code sufficient to reproduce the main experiments, which can be found here at \url{https://github.com/AndrewRqy/No-False-Positive-Paper-ver/}.

\bibliography{references}
\bibliographystyle{plainnat}

\newpage

\appendix
\section*{LLM Usage}
Claude Opus 4.6 was used to assist with the initial ideation of the NFP test, retrieval of related work, and proving and verification of mathematical claims in Appendices~\ref{app:generalNFPframework}, \ref{app:ellp>1}, \ref{app:decorrelatestimulus} and \ref{app:saturation}. We also used Claude Opus 4.8 and Claude Fable 5 to assist with adapting the SAE training toolkit to V-JEPA2 and TimeSformer, standardizing causal analysis interventions across different models, and constructing tables for the paper, as well as putting together the codebase found at \url{https://github.com/AndrewRqy/No-False-Positive-Paper-ver/}. These were carefully checked by the authors. LLMs were also used for advice on writing style and language polishing, as well as initial draft sections. All final writing and the main mathematical results in this paper are the result of authors' writing and editing, and are not generated directly by LLM.

\section{A General NFP Framework}
\label{app:generalNFPframework}

The NFP test developed in Section~\ref{section:theoreticalframework} addresses a specific instance of a more general problem: determining whether a feature depends on a target concept or merely on a confounding variable that happens to co-vary with it. In our paper, the target concept is a temporal property (e.g.\ speed), and the confounding variable is ball position, which is physically coupled with motion.

In this appendix we state and prove a general NFP theorem, show that the result of Section~\ref{section:theoreticalframework} is a specific instance, and outline how the framework applies to problems outside of video model interpretability.

\subsection{Setup}
\label{app:general_setup}

\begin{definition}[Observational structure]
An NFP problem consists of:
\begin{itemize}[leftmargin=*, itemsep=2pt, topsep=2pt]
    \item $N$ \textbf{units} indexed by $n = 1, \ldots, N$. Each unit can be thought of as a single experiment or trial.
    \item $T$ indexed observations per unit, $t = 1, \ldots, T$. 
    \item A \textbf{representation} $h(n, t) \in \mathbb{R}^d$. This can be thought of as some system's output at $n$, $t$.
    \item A \textbf{target concept profile} $\tau(n, t) \in \mathbb{R}$: the value of the concept of interest at $n,t$.
    \item A \textbf{confounding variable} $C(n, t) \in \mathscr{C}$: a quantity that co-varies with $\tau$ within each unit, where $\mathscr{C}$ is an arbitrary space.
\end{itemize}
\end{definition}

The core difficulty is that $C$ and $\tau$ change together across observations within each unit, so a feature that appears to track $\tau$ may instead be responding to $C$.

\begin{definition}[Confound-dependent feature]
\label{def:confound_dependent}
A scalar feature $f$ is \emph{confound-dependent} if there exists a function $g\colon \mathscr{C} \to \mathbb{R}$ such that for all $n,t$,
\begin{equation}
    f(n, t) = g(C(n, t))
\end{equation}
\end{definition}

\begin{definition}[Within-unit covariance]
\label{def:within_unit_cov}
For a scalar feature $f$, the within-unit covariance with $\tau$ at unit $n$ is
\begin{equation}
    \mathcal{C}_f(n) 
    = \frac{1}{T}\sum_{t=1}^{T} \bigl(f(n,t) - \bar{f}(n)\bigr)\bigl(\tau(n,t) - \bar{\tau}(n)\bigr)
\end{equation}
where $\bar{f}(n) = \frac{1}{T}\sum_t f(n,t)$ and $\bar{\tau}(n) = \frac{1}{T}\sum_t \tau(n,t)$.
\end{definition}

The NFP test asks whether the sample mean $\bar{\mathcal{C}}_f = \frac{1}{N}\sum_n \mathcal{C}_f(n)$ is significantly different from zero.

\subsection{Assumptions}

The following three assumptions define what it means for the experimental design to ``decouple'' the confounder from the target concept.

\begin{description}[leftmargin=0pt, itemsep=6pt, style=unboxed]

\item[Assumption 1 (Randomisable confounder).]
There exists a random variable $R(n)$, taking values in some measurable space $\mathcal{R}$, and for each $n$ a deterministic function $\phi_n\colon \mathcal{R} \times \{1,\ldots,T\} \to \mathscr{C}$ (not to be confused with $\phi_w$) such that,
\begin{equation}
    C(n, t) = \phi_n(R(n), t)
\end{equation}
The confounding variable at every condition is determined by a per-unit random seed $R$ and the observation index $t$.

\item[Assumption 2 (Independence).]
The random seed $R(n)$ is independent of the target concept profile $\tau(n, t)$.

\item[Assumption 3 (Equidistribution / Coverage).]
For each $n$, the distribution of $C(n,t) = \phi_n(R(n), t)$ over the random variable $R$ is the same for every $t$. 
\end{description}

The intuition behind Assumption~3 is the same as in the main paper: if the confounder has the same distribution at every $t$, then any feature that depends only on the confounder will have the same expected value at every condition, and a value that is constant in expectation across conditions cannot covary with anything. In practice, we need $\mathbb{E}[|g(\phi(R,t))|]$ to be well-defined. This is automatically satisfied whenever $g$ is bounded, which holds for neural network features.

Different applications achieve Assumption~3 through different mechanisms. In the main paper, it is achieved through uniform starting position, off-screen zeroing, and the coverage condition (S1, S4, S5). Other mechanisms, such as periodic (toroidal) boundary conditions, can likewise achieve Assumption 3.

\subsection{Theorem and proof}

\begin{theorem}[General NFP]
\label{thm:general_nfp}
Under Assumptions 1--3, for any confound-dependent feature $f(n,t)=g(C(n,t))$,
\begin{equation}
    \mathbb{E}\bigl[\mathcal{C}_f(n)\bigr] = 0
\end{equation}
\end{theorem}

\noindent\textbf{Contrapositive.} \textit{If $\bar{\mathcal{C}}_f$ is significantly nonzero, then $f$ is not confound-dependent: $f$ must depend on something beyond the confounding variable.}

\begin{proof}
We follow the proof structure in Section~\ref{section:theoreticalframework}.
We first note that the within-unit covariance can be simplified:
    $$\mathcal{C}_f(n)
    = \frac{1}{T}\sum_{t=1}^T f(n,t)\bigl(\tau(n,t)-\bar\tau(n)\bigr)
       - \bar f(n)\cdot\underbrace{\frac{1}{T}\sum_{t=1}^T\bigl(\tau(n,t)-\bar\tau(n)\bigr)}_{=\;0}$$
The second term vanishes because deviations from the mean sum to zero. We get,
$$\mathcal{C}_f(n) = \frac{1}{T}\sum_{t=1}^T f(n,t)\bigl(\tau(n,t)-\bar\tau(n)\bigr)$$
Now substitute the confound-dependent structure $f(n,t) = g(\phi_n(R(n), t))$ and take the expectation over $R$. By Assumption~2, $\tau(n,\cdot)$ is independent of $R(n)$, so the factors $\tau(n,t) - \bar\tau(n)$ are constant with respect to the expectation over $R$:
$$\mathbb{E}_R\bigl[\mathcal{C}_f(n)\bigr] = \frac{1}{T}\sum_{t=1}^T \mathbb{E}_R\bigl[g(\phi_n(R,t))\bigr]\cdot\bigl(\tau(n,t)-\bar\tau(n)\bigr)$$
By Assumption~3, $\mathbb{E}_R[g(\phi_n(R,t))] = \mu_g^{(n)}$ for all $t$, so
$$\mathbb{E}_R\bigl[\mathcal{C}_f(n)\bigr] = \frac{\mu_g^{(n)}}{T}\sum_{t=1}^T \bigl(\tau(n,t)-\bar\tau(n)\bigr) = \mu_g^{(n)} \cdot 0 = 0$$
This holds for every fixed target concept profile $\tau(n,\cdot)$. If the profiles are themselves randomly sampled (as when different units have different velocity profiles), the law of total expectation over $\tau$, i.e. over $n$, gives:
$$\mathbb{E}\bigl[\mathcal{C}_f(n)\bigr] = \mathbb{E}_\tau\!\Bigl[\mathbb{E}_R\bigl[\mathcal{C}_f(n)\bigr]\Bigr] = \mathbb{E}_n[0] = 0$$
\end{proof}

\subsection{The paper's setup as an instance of the general framework}

The NFP test of Section~\ref{section:theoreticalframework} is recovered by the following instantiation:

\begin{center}
\small
\begin{tabular}{@{}ll@{}}
\toprule
\textbf{General framework} & \textbf{This paper} \\
\midrule
Unit $n$ & Video $V$ \\
Observation $t$ & Time step $t$ \\
Representation $h(n,t)$ & Layer-$l$ representation at ball token \\
Target concept $\tau(n,t)$ & Temporal concept (speed, $v_x$, direction, etc.) \\
Confounding variable $C(n,t)$ & Effective ball position: $(x_t,y_t)$ if on-screen, $\bot$ if off-screen \\
Random seed $R(n)$ & Starting position $(x_0, y_0) \sim \mathrm{Uniform}([a_x,b_x]\times[a_y,b_y])$ \\
$\phi(R,t)$ & $(x_0+s_x(t),\;y_0+s_y(t))$ if in $[0,H]^2$, else $\bot$ \\
Assumption 1 & Ball position determined by starting position + velocity profile \\
Assumption 2 & Starting position independent of velocity profile (S3) \\
Assumption 3 & Uniform starting position (S1) + coverage (S5) + off-screen zeroing (S4) \\
\bottomrule
\end{tabular}
\end{center}

For assumption 3, in our video model with ball videos setup, the raw ball position $(x_t, y_t) \in \mathbb{R}^2$ does \emph{not} have a $t$-independent distribution because the uniform rectangle shifts with the cumulative displacement $s(t)$. However, the \emph{effective} confounding variable, i.e. ball position when on-screen, an invalid default $\bot$ when off-screen, does have a $t$-independent distribution, because:

\begin{enumerate}[leftmargin=*, itemsep=2pt]
    \item The coverage condition (S5) guarantees that the visible screen $[0,H]^2$ is contained within the shifted rectangle $[a_x+s_x(t),\, b_x+s_x(t)] \times [a_y+s_y(t),\, b_y+s_y(t)]$ for every $t$ and every velocity profile.
    \item Therefore, the probability that the ball is on-screen is the same regardless of $t$, and conditional on being on-screen, the position is $\mathrm{Uniform}([0,H]^2)$ regardless of $t$.
    \item Off-screen zeroing (S4) sets all features to output zero when the ball is off-screen. For any confound-dependent feature $f(n,t) = g(C(n,t))$, this means $f=g(\bot) = 0$.
\end{enumerate}

\noindent Together, these ensure that $\mathbb{E}_R[g(\phi(R,t))]$ is the same for all $t$ for every confound dependent feature, satisfying Assumption~3 and enabling the general theorem.

\subsection{Multi-confounder extension}

It's worth asking if the NFP framework extends to settings with multiple confounding variables. In the main paper, ball position $(x_t, y_t)$ is the sole confounder, but we can imagine setups with multiple moving objects, where a feature could depend on the joint
positions of all objects.

So we ask: if a feature depends on several confounders jointly, under what conditions does the NFP property still hold?

\textbf{Independent seeds.} Suppose we have $K$ confounding variables $C_1(n,t), \ldots, C_K(n,t)$, each driven by its own random seed: $C_k(n,t) = \phi_{k,n}(R_k(n), t)$. A feature $f$ is \emph{jointly confound-dependent} if
$$f(n,t) = g\bigl(C_1(n,t), \ldots, C_K(n,t)\bigr)$$
for some function $g : \mathscr{C}_1 \times \cdots \times
\mathscr{C}_K \to \mathbb{R}$.

If the seeds $R_1, \ldots, R_K$ are mutually independent, each independent of $\tau$, and each satisfying Assumption~3 individually (i.e.\ for each $k$ and $n$, the distribution of $C_k(n,t)$ over $R_k$ is the same for all $t$), then the NFP property holds for any jointly confound-dependent feature.

The reason is that independent seeds with individually
$t$-independent distributions give a jointly $t$-independent distribution: the joint distribution of $(C_1(n,t), \ldots, C_K(n,t))$ is the product of the marginals (by independence of the
seeds), and each marginal is the same for every $t$, so the product is the same for every $t$. This reduces to the single-confounder case with $C = (C_1, \ldots, C_K)$ and $R = (R_1, \ldots, R_K)$, and Theorem~\ref{thm:general_nfp} applies directly.

\textbf{An example of an application is to multiple moving objects.} Consider a video with two balls, Ball A and Ball B, each with its own starting position and velocity profile. A researcher wants to test whether a feature encodes Ball A's speed ($\tau$), or whether it merely responds to the positions of both balls (confounders $C_1 = (x_t^A, y_t^A)$, $C_2 = (x_t^B, y_t^B)$). If the starting positions $(x_0^A, y_0^A)$ and $(x_0^B, y_0^B)$ are drawn independently of each other and independently of both velocity profiles, with each satisfying the coverage condition, then the joint distribution is $t$-independent, and the NFP property holds: any feature depending only on the joint position, including relational features like ``balls are close,'' has zero expected covariance with Ball A's speed.

\textbf{Independent seeds are necessary, not just convenient.} If the confounders share a seed, the NFP property can fail even when each confounder individually has a $t$-independent distribution. Consider a simple example: two confounders $C_1, C_2$
both driven by the same seed $R \sim \mathrm{Uniform}[0, L)$ on a periodic (toroidal) domain, with $C_1(n,t) = (R + s_1(t)) \bmod L$ and $C_2(n,t) = (R + s_2(t)) \bmod L$, where $s_1, s_2$ are different displacement functions. Each marginal is $\mathrm{Uniform}[0, L)$ for all $t$, so each individually satisfies Assumption~3. However, the difference $C_1 - C_2 = (s_1(t) - s_2(t)) \bmod L$ is deterministic and varies with $t$. A feature like $g(C_1, C_2) = \mathbf{1}(|C_1 - C_2| < \epsilon)$,
which fires when the two confounders are close, has $t$-dependent expected activation. Such a feature could be flagged as temporal by the NFP test despite being purely confound-dependent. In this case the joint distribution shifts with $t$, even though the marginals do not. 

Applying such a failure case to the previous example, we consider the jointly confound-dependent feature ``the two balls are close together,'' i.e.\ $g(C_1, C_2) = \mathbf{1}(\|C_1 - C_2\| < r)$. If $C_1, C_2$ share a seed, this feature's activation varies with time and may covary with Ball A's speed, even though it depends only on position.

\subsection{Examples in other domains and remark about stochastic confounders}

To illustrate the generality of the framework, we outline how it would apply to three problems outside video interpretability. In each case, the same design idea applies: randomise a seed to give the confounder the same distribution at every $t$, then test within-unit covariance.

\textbf{Neuroscience:}
A neuroscientist records from a neuron while an animal tracks a moving stimulus. The question is whether the neuron encodes stimulus velocity ($\tau$) or merely responds to stimulus position ($C$), which co-varies with velocity within each trial. Units are trials ($n$), observations are time steps $t$, and the representation is the neuron's firing rate. By randomising the starting stimulus position across trials and ensuring the position distribution covers the neuron's receptive field at every time step $t$ (the coverage condition), any position-dependent response has zero expected covariance with velocity. Significant covariance then proves the neuron encodes something beyond position.

\textbf{Manufacturing: sensor attribution.}
A process engineer wants to determine whether a quality sensor responds to production speed ($\tau$) or merely to temperature ($C$), where speed and temperature co-vary during each production run. If the starting temperature can be randomised independently of the speed profile, and the randomisation ensures the temperature distribution at each time $t$ is the same, the framework can apply. A sensor whose readings depend only on temperature will show zero expected covariance with speed; significant covariance proves the sensor responds to something beyond temperature.

\textbf{Finance: signal attribution.}
A quantitative analyst wants to know whether a trading signal responds to momentum ($\tau$) or merely to price level ($C$). Within any evaluation window, price level changes with momentum, creating the same coupling problem. Units are evaluation windows, observations are time steps. Price at step $t$ is $p_0 + s(t)$ where $s(t)$ is the cumulative return, so price level has exactly the form $\phi(R, t) = R + s(t)$. To get equidistribution over $t$, we can randomise the entry price $p_0$ uniformly over a range wide enough that the signal's ``active zone'' (the analogue of the visible screen) is always covered. A signal that depends only on price level will show zero expected covariance with momentum.

Overall, we also note without proof that the NFP framework may extend, with modifications, to settings where the confounder has stochastic dynamics within each unit. We suspect that not all kinds of stochastic dynamics can be accommodated into a modified NFP framework. A full treatment is beyond the scope of this paper.

\section{Implications of $\ell_p>1$}
\label{app:ellp>1}

\subsection{Theoretical implications}

When $\ell_p > 1$, each token (tubelet) covers $\ell_p$ consecutive frames. This means a tubelet's visual content depends on the ball's position at all $\ell_p$ frames it covers. This raises the question: what counts as a static feature when a token already spans multiple frames?

We should just choose a confounder based on intuitive logic, e.g. average ball position within each tubelet, or ball position on a consistent select frame within each tubelet. This does work, and we show the proof goes through.

\textbf{What the proof requires.} We revisit why the proof works in the $\ell_p = 1$ case. At frame $t$, the ball position is $x_t = x_0 + s_x(t)$, and a static feature satisfies $\psi_w(V,t) = G(x_t, y_t)$. The coverage argument shows that $\mathbb{E}_{x_0}[G(x_0 + s_x(t),\; y_0 + s_y(t))]$ is independent of $t$. This works because at every $t$, $x_0 + s_x(t)$ is a shift of the uniform random variable $x_0$ by some constant, and a uniform distribution shifted by a constant remains uniform over a shifted interval. The coverage condition ensures $G$'s support is always inside this interval, so the integral is always the same. The important thing to notice is that this argument operates on $x_0$. If $x_0$ were absent from the quantity, the coverage argument would have nothing to act on.

We consider the general case of a linear combination of multiple positions. For $\ell_p > 1$, let tubelet $k$ cover frames $t_{k,1}, \ldots, t_{k,\ell_p}$. The ball's position at each frame is $x_{t_{k,j}} = x_0 + s_x(t_{k,j})$ and analogously for $y$. Consider a feature that depends on a linear combination of these positions:
$$\sum_{j=1}^{\ell_p} \alpha_j\, x_{t_{k,j}} = \Bigl(\sum_{j=1}^{\ell_p} \alpha_j\Bigr) x_0 + \sum_{j=1}^{\ell_p} \alpha_j\, s_x(t_{k,j})$$
The coefficient on $x_0$ is $\alpha=\sum_j \alpha_j$, and the remaining term is a constant $c_k$ determined by the velocity profile and tubelet index. Two cases arise.

\textbf{If $\alpha \neq 0$}, then $x_0$ survives, and for each time step $k$ we get, given appropriate coverage,
$$\mathbb{E}_{x_0}[G(\alpha\, x_0 + c_k)] = \frac{1}{\alpha(b-a)}\int_{\alpha a+c_k}^{\alpha b+c_k} G(x)\,dx$$
and the proof goes through. Here it's important to understand that the NFP framework works on the user choosing a confounder first to define $f(n,t)=g(C(n,t))$ as per Appendix~\ref{app:generalNFPframework}, then we compute the coverage condition just once on the confounder we chose. The coverage is standard for $\alpha=1$. This is harder to digest for $\alpha\neq1$, but it follows condition 3 from Appendix~\ref{app:generalNFPframework}. Any reasonable choice of confounder in the context of this paper has $\alpha=1$.

\textbf{If $\alpha = 0$}, then $x_0$ cancels entirely, and the quantity reduces to the constant $c_k$. For example, the within-tubelet displacement is $x_{t_{k,2}} - x_{t_{k,1}} = s_x(t_{k,2}) - s_x(t_{k,1})$, which is a fixed value determined by the velocity profile. Then $\mathbb{E}_{x_0}[G(c_k)]$ is $k$-dependent and so $t$ dependent. Such a feature can be flagged by the NFP test, which is correct, since it encodes within-tubelet motion and is genuinely temporal. This is not what we want, since the confounder is supposed to define static features. We also observe that such a confounder is intuitively nonsensical, and we explore the theory here just for the sake of good understanding.

We also claim without proof that confounders that are nonlinear functions of multiple ball-positions within a tubelet can in general produce features flagged as temporal by the NFP test. Such confounders are also nonsensical in our context, so we omit detailed discussion.

\subsection{Experimental implications}

Since frames within the same tubelet share the same activation $\psi_w(V,k)$, there are two options for computing the within-video covariance. First, we could maintain $T$ terms, using the same activation for all frames in a tubelet. The NFP proof holds unchanged. Second, we could compute covariance using $T/\ell_p$ terms by averaging the temporal concept profile values within each tubelet. This is what we do empirically in this paper, using $T = 8$ effective time steps instead of $T = 16$ frames. The two options are mathematically equivalent for the calculation of $C_w(V)$, so either works.

We note that selecting the concept profile value $\tau$ for a single frame of each tubelet instead of averaging may produce errors in downstream analysis. For instance, consider $\tau = [0, 10, 0, 10, 10, 0, 10, 0, 0, 10, 0, 10, 10, 0, 10, 0]$ with $\ell_p = 2$. Each tubelet has the same summed value, but selecting only the first frame gives $[0, 0, 10, 10, 0, 0, 10, 10]$ while selecting the second gives $[10, 10, 0, 0, 10, 10, 0, 0]$, neither of which reflects the true structure. Averaging gives $[5, 5, 5, 5, 5, 5, 5, 5]$, which correctly represents that all tubelets have equal temporal content.

\section{Compact support and Coverage condition}
\label{app:compactsupportcoveragecondition}

\subsection{More on compact support and coverage condition}

There are two main things that are useful to note.

First, in the main text we take the visible frame $[0,H]^2$ as the region outside which all activations $\psi_w(V,t)=0$. The proof holds more generally: let $\Omega \subset \mathbb{R}^2$ be any bounded region that describes the support of $G$, outside which $\psi_w=0$. Then necessarily we have $\Omega \subseteq [0,H]^2$. The coverage condition then requires $a_x \leq \Omega_x^- - s_x^{\max}$, $b_x \geq \Omega_x^+ + s_x^{\max}$, and analogously for $y$. We observe $[0,H]^2$ is the maximal such $\Omega$, and that it's the most intuitive use case. The proof goes through simply by rewriting the range of the integral as $\Omega$,
$$\mathbb{E}_{x_0, y_0}[G(x_t, y_t)] = \frac{1}{L_xL_y} \iint_{\Omega} G(x, y)\, dy\, dx = \mu_G$$
where $L_x=b_x-a_x, L_y=b_y-a_y$. Here the value of $\mu_G$ may change depending on $\Omega$.

Second, for S4, any $\psi_w(V,t)=c$ will work for $(x_t,y_t)\notin\Omega$, not only $\psi_w(V,t)=0$. The proof follows because we get:
$$\mathbb{E}[G(x_t, y_t)] = \frac{1}{L_x L_y}\left[\iint_\Omega G(x,y)\, dy\, dx + c \cdot (L_x L_y - \text{Area}(\Omega))\right]$$ 
Here the first term is what we get in the original proof and reduces to $\mu_G$, and the second term is clearly also independent of $t$, so the entire expression reduces to a constant, which is what we need. So for (S4) what we actually need is not compact support, but that $\psi_w$ is constant outside $\Omega$. Compact support is the special case regarding $G$ where $c = 0$. We use compact support nonetheless because it is the most natural choice, and is easy to describe and understand.

\subsection{Empirical procedure for implementation of (S1) - (S5) and the NFP test}

1. \textbf{Construct the stimulus set.} First, choose a (possibly diverse) set of velocity profiles to use. Then mathematically simulate in mathematical coordinates the required spawning range $[a_x,b_x]\times[a_y,b_y]$ satisfying the coverage condition. Then simulate the $N$ ball movements by assigning motion profiles, with starting positions drawn independent of motion profile and uniformly over the spawning range. From the $N$ simulations, generate $N$ synthetic videos, calculating the pixel values of the screen from ball positions in the simulations. This can generate 4 kinds of videos. All 4 kinds are to be kept in the stimulus set and used in the subsequent steps:

1. Ball starts on-screen, stays on-screen.\\
2. Ball starts off-screen, enters screen.\\
3. Ball starts on-screen, exits screen.\\
4. Ball starts off-screen, never enters screen.

2. \textbf{Run the model}. Process each video through the frozen video transformer. For each direction $w$ and video $V$, extract the ball-tracking activation time series $\psi_w(V, t)$. For frames where the center of the ball is not on-screen based on the mathematical coordinates from step 1, set $\psi_w(V,t)=0$.

3. \textbf{Choose a temporal concept profile $\tau$}. For each video $V$, compute $\tau(V, t)$ from the synthetic data generated in step 1 which is ground truth.

4. \textbf{Compute the test statistic}. For each video $V$, compute $C_w(V) = \text{Cov}_t(\psi_w(V, t), \tau(V, t))$.

5. \textbf{Test for significance}. Perform a one-sample $t$-test on $\{C_w(V_1), \ldots, C_w(V_N)\}$ against zero. If significant at level $\alpha$, direction $w$ is flagged as non-static.

\section{Counterexample Against Correlation}
\label{app:corrcounterexample}

We show that the no-false-positives property does \emph{not} hold if Pearson correlation is used in place of covariance. We construct a static feature for which the expected covariance is zero but the expected Pearson correlation is nonzero, so a correlation-based test would falsely flag it as temporal.

The Pearson correlation is
\begin{equation}
    \rho_w(V) = \frac{C_w(V)}{\sqrt{\mathrm{Var}_t(\psi_w(V,t))\;\mathrm{Var}_t(\tau(V,t))}}
\end{equation}
We fix the velocity profile across videos, so $\mathrm{Var}_t(\tau)$ is a constant.

We work in one spatial dimension; the counterexample extends to two dimensions by letting the feature depend only on $x$. Let $T = 2$ with velocity profile $v_1 = 0,\; v_2 = 2$, giving cumulative displacements $s_1 = 0,\; s_2 = 2$ and positions $x_1 = x_0,\; x_2 = x_0 + 2$. The velocity mean is $\bar{v} = 1$ and $\mathrm{Var}_t(v) = 1$. For $T = 2$, the within-video covariance simplifies to $C_w(V) = \tfrac{1}{2}(\psi_2 - \psi_1)$.

Define a static feature:
\begin{equation}
    g(x) =
    \begin{cases}
        2 & x \in [0, 1] \\
        1 & x \in [2, 3] \\
        0 & \text{otherwise}
    \end{cases}
\end{equation}
Let $x_0 \sim \mathrm{Uniform}[-W, W]$ with $W$ large enough to satisfy the coverage condition. Then $\psi_{t=1} = g(x_0),\; \psi_{t=2} = g(x_0 + 2)$. For most values of $x_0$ both activations are zero, giving $C_w = 0$ and $\rho_w$ undefined. The three intervals where at least one activation is nonzero are summarised below.
\begin{center}
\begin{tabular}{@{}llccccc@{}}
\toprule
& $x_0$ interval & $\psi_{t=1}$ & $\psi_{t=2}$ & $C_w$ &
    $\mathrm{Var}_t(\psi)$ & $\rho_w$ \\
\midrule
Case A & $[-2,\, -1]$ & 0 & 2 & $+1$ & 1 & $+1$ \\
Case B & $[0,\, 1]$ & 2 & 1 & $-\tfrac{1}{2}$ & $\tfrac{1}{4}$ & $-1$ \\
Case C & $[2,\, 3]$ & 1 & 0 & $-\tfrac{1}{2}$ & $\tfrac{1}{4}$ & $-1$ \\
\bottomrule
\end{tabular}
\end{center}
All three intervals have equal width, so they are equally likely. The expected covariance over these intervals is $\mathbb{E}[C_w] \propto (+1) + (-\tfrac{1}{2}) + (-\tfrac{1}{2}) = 0$, following Theorem~1. However, the expected correlation conditional on being defined is
\begin{equation}
    \mathbb{E}[\rho_w \mid \text{defined}]
    = \frac{(+1) + (-1) + (-1)}{3}
    = -\frac{1}{3}
    \;\neq\; 0
\end{equation}
The mechanism is that the denominator $\sqrt{\mathrm{Var}_t(\psi_w)}$ is position-dependent: Case~A ($x_0 \in [-2,-1]$) has a large activation variance that shrinks its positive covariance, while Cases~B and~C have small activation variance that amplifies their negative covariance. This position-dependent rescaling breaks the symmetry that makes the expected covariance zero, producing a nonzero expected correlation for a purely static feature.

\section{Decorrelated Stimulus}
\label{app:decorrelatestimulus}

This appendix is motivated by \textbf{L1. The NFP test detects non-static features, not specific temporal concepts.} For directions $w$ that pass the test, we cannot conclude they encode $\tau$ specifically, because other temporal properties may covary with $\tau$ within each video. A natural follow-up question is whether we are able to recover the true encoding directions of concept $\tau$ if they exist. This requires decorrelated stimulus, which recovers true encoding directions only if decorrelation is perfect. We provide a mathematical framework for decorrelated stimulus in this Appendix, and show that on real video models within our framework, not only is it not possible to recover in practice the true encoding directions for temporal concepts, but that even a $w$ that gives the strongest covariance for some concept $\tau$ can be orthogonal to the true encoding direction $d_{\tau}$. Thus, trying to say that a feature encodes a specific $\tau$ is likely not a meaningful statement in the NFP framework. 

We structure this Appendix in a pedagogical manner to help readers understand the thought process behind our argument. For how decorrelation is used in practice in the context of the paper, see Appendix~\ref{app:decorrelation_practice}.

\subsection{The mean covariance vector $\bar{c}_\tau$}
\label{app:cbar_derivation}

We work with $F$ as the identity throughout this appendix, so $\psi_w(V,t) = w^T h(V,t)$ where $h(V,t) = r^l_{b(V,t)}(V)$ is the representation at the ball-tracking token. Unlike the NFP test, which applies to arbitrary deterministic $F$ including SAE nonlinearities, the derivation of $\bar{c}_\tau$ requires linearity in $w$ to factor the direction out of the covariance.

We apply from Section~\ref{section:theoreticalframework} the within-video covariance between $\psi_w$ and a temporal concept profile $\tau$:
$$C_w(V) = \frac{1}{T}\sum_{t=1}^{T}(w^T h(V,t) - w^T\bar{h}(V))(\tau(V,t) - \bar{\tau}(V))$$
Since $w^T$ is linear, it factors out of the sum:
$$C_w(V) = w^T \underbrace{\frac{1}{T}\sum_{t=1}^{T}(h(V,t) - \bar{h}(V))(\tau(V,t) - \bar{\tau}(V))}_{\;=\; c_\tau(V)\;\in\;\mathbb{R}^d}$$
The vector $c_\tau(V)$ is the within-video covariance between the full representation and $\tau$. Its $i$-th component is the scalar covariance $C_{e_i}(V)$ along the $i$-th dimension axis. Using $\sum_t (\tau(V,t) - \bar{\tau}(V)) = 0$, this simplifies to:
\begin{equation}
    c_\tau(V) = \frac{1}{T}\sum_{t=1}^{T}
    h(V,t)\,(\tau(V,t) - \bar{\tau}(V))
\end{equation}
Taking the expectation over videos:
\begin{equation}
    \bar{c}_\tau = \mathbb{E}_V[c_\tau(V)], \qquad
    \mathbb{E}_V[C_w(V)] = w^T \bar{c}_\tau
    \label{eq:cbar}
\end{equation}
\textbf{The NFP property carries over to the vector.} Suppose the entire representation is static, i.e. $h(V,t) = H(x_t(V), y_t(V))$ for some function $H:\mathbb{R}^2\to\mathbb{R}^d$. Then for every direction $w$, the scalar $w^T h(V,t) = w^T H(x_t, y_t)$ is a static function of position, so by Theorem~\ref{thm:main}, $w^T \bar{c}_\tau = \mathbb{E}[C_w(V)] = 0$. Since this holds for all $w$, we have $\bar{c}_\tau = \mathbf{0}$. By contrapositive, $\bar{c}_\tau \neq \mathbf{0}$ implies that the representation carries temporal information.

\textbf{Maximally covarying direction.} By Cauchy-Schwarz,
$w^T \bar{c}_\tau \leq \|w\|\,\|\bar{c}_\tau\|$ with equality when $w = \bar{c}_\tau / \|\bar{c}_\tau\|$. So $\bar{c}_\tau / \|\bar{c}_\tau\|$ is also the unit direction whose activation maximally covaries with $\tau$.

\subsection{Structured decorrelation}
\label{app:structured_decorrelation}

The vector $\bar{c}_\tau$ gives the direction $w$ that maximally covaries with $\tau$, but does it point along the direction that \emph{encodes} $\tau$ in the model? Under additional assumptions, we can make this precise.

\textbf{Additive linear encoding.} Suppose the representation at the ball-tracking token decomposes as:
\begin{equation}
    h(V,t) = \sum_{k=1}^{K} \tau_k(V,t)\, d_k + \text{static}(V,t) 
    \label{eq:additive}
\end{equation}
where $\tau_1, \ldots, \tau_K$ are temporal concepts encoded along directions $d_1, \ldots, d_K \in \mathbb{R}^d$, and $\text{static}(V,t)$ depends only on the ball's current position. Substituting into $\bar{c}_\tau$ and using the bilinearity of covariance and linearity of expectation, we get:
$$c_\tau(V) = \frac{1}{T}\sum_{t=1}^{T} \left[\sum_{k=1}^{K} \tau_k(V,t)\,d_k + \text{static}(V,t)\right](\tau(V,t) - \bar{\tau}(V))$$
$$c_\tau(V) = \sum_{k=1}^{K} d_k \cdot \frac{1}{T}\sum_{t=1}^{T} \tau_k(V,t)(\tau(V,t) - \bar{\tau}(V)) + \frac{1}{T}\sum_{t=1}^{T} \text{static}(V,t)(\tau(V,t) - \bar{\tau}(V))$$
\begin{equation}
    \bar{c}_\tau = \sum_{k=1}^{K} d_k \cdot \mathbb{E}_V[\text{Cov}_t(\tau_k, \tau)]
    \label{eq:cbar_decomp}
\end{equation}
The static terms vanish by the NFP property (Theorem~\ref{thm:main}). If we choose $\tau = \tau_1$ (say, velocity), then:
\begin{equation}
    \bar{c}_{\text{vel}} = d_1 \cdot\mathbb{E}[\text{Var}_t(\text{vel})] + \sum_{k=2}^{K} d_k\cdot \mathbb{E}[\text{Cov}_t(\tau_k, \text{vel})]
\end{equation}
The first term points along $d_1$, the true velocity encoding direction. The remaining terms point along other encoding directions and contribute when $\tau_k$ covaries with velocity in videos, even if the two temporal concepts are \emph{semantically} distinct (e.g. velocity and acceleration). This causes $\bar{c}_{\text{vel}}$ to point away from $d_1$.

\textbf{The decorrelation idea.} If we design the stimulus set so that $\mathbb{E}_V[\text{Cov}_t(\tau_k, \text{vel})] = 0$ for all $\tau_k\neq\text{vel}$, the cross-terms vanish:
$$\bar{c}_{\text{vel}} = d_1 \cdot\mathbb{E}[\text{Var}_t(\text{vel})]\;\propto\; d_1$$
and the $\bar{c}_\tau$ recovers the true velocity encoding direction. More generally, pairwise decorrelation of all $K$ concepts ($\mathbb{E}[\text{Cov}_t(\tau_j, \tau_k)] = 0$ for all $j \neq k$) gives $\bar{c}_{\tau_k} \propto d_k$ for every concept simultaneously.

This recovery, however, depends on two assumptions: the additive linear encoding Equation~\ref{eq:additive} and a complete, correct list of concepts $\tau_1, \ldots, \tau_K$. We return to the consequences of violating these assumptions in Section~\ref{app:decorrelation_limits}.

\subsection{Obstructing concepts}
\label{app:decorrelation_obstructing}

Not every temporal concept can be decorrelated from velocity. We say a concept $\tau_k$ is \textbf{obstructing} if, for every non-constant velocity profile $\pi$, the within-video covariance $\phi_{\tau_k}(\pi)=\text{Cov}_t(\tau_k^{(\pi)}, \text{vel}^{(\pi)})$ has the same strict sign (not to be confused with $\phi_w$). When this holds, no mixture of velocity profiles can produce zero expected covariance, because a weighted average of values that are all strictly positive (or all strictly negative) cannot be zero.

Obstructing concepts can arise through many mechanisms; we list two relevant for our discussion:

\textbf{Pointwise monotone functions.} If
$\tau_k(t) = f(v(t))$ for a strictly monotone increasing function $f$, then Chebyshev's sum inequality gives $\text{Cov}_t(f(v), v) \geq 0$ for any sequence $(v_1, \ldots, v_T)$, with equality only when all $v_t$ are equal. The intuition is that monotonicity preserves ordering. $f(v_t)$ is large exactly when $v_t$ is large, so they positively covary. Examples include $f(v) = v$ (velocity itself), $f(v) = v^3$, and $f(v) = e^v$. A similar argument applies for strictly monotone decreasing $f$.

Non-monotone functions such as $v^2$, $|v|$, or $e^{v^2}$ can be non-obstructing, because velocity profiles exist that produce both positive and negative covariance with these functions.

\textbf{Linear temporal concepts.} A concept of the form
$\tau_k(t) = \sum_s M_{ts}\, v(s)$, where $M \in \mathbb{R}^{T\times T}$. Here $t, s \in[1,T]$ are simply time steps, and the matrix $M_{ts}$ provides the coefficients. For example, we can have for $T=3,t = 2$ something like $\tau_k(2) = M_{2,1}v(1) + M_{2,2}v(2) + M_{2,3}v(3)$. 

Let $P_c = I - \frac{1}{T}\mathbf{1}\mathbf{1}^T$ be the centering matrix. Writing any velocity profile as $\pi = \bar{v}\mathbf{1} + \tilde{\pi}$ where $\tilde{\pi} = P_c \pi$ is the mean-centered part, the within-video covariance becomes:
\begin{equation}
    \phi_{\tau_k}(\pi) = \frac{1}{T}(\tilde{\pi}^T M^T \tilde{\pi} + \bar{v} \cdot \mathbf{1}^T M^T \tilde{\pi})
\end{equation}
using $P_c \mathbf{1} = 0$ and $P_c \tilde{\pi} = \tilde{\pi}$. The concept is obstructing if and only if both conditions hold:

(a) Cross-term elimination i.e. $\mathbf{1}^T M^T \tilde{\pi} = 0$ for all $\tilde{\pi} \in \text{col}(P_c)$.

If this fails, there exists $\tilde{\pi}_0$ with $\mathbf{1}^T M^T \tilde{\pi}_0 \neq 0$. The expression $\bar{v} \cdot \mathbf{1}^T M^T \tilde{\pi}_0 + \tilde{\pi}_0^T M^T \tilde{\pi}_0$ is linear in $\bar{v}$, so by choosing $\bar{v}$ with appropriate sign and large magnitude, it can be made to take either sign. Hence the concept becomes non-obstructing.

This is equivalent to $M\mathbf{1} \in \text{span}(\mathbf{1})$, i.e., $\mathbf{1}$ is an eigenvector of $M$, because we need $\mathbf{1}^T M^T \tilde{\pi} =(M\mathbf{1})^T\tilde{\pi}= 0$, so the vector $M\mathbf{1}$ must be orthogonal to every $\tilde{\pi}$. Since $\tilde{\pi}$ is a mean zero vector, this means that $M\mathbf{1}$ must be a constant vector, which is ones vector multiplied by a scalar $c$.

(b) The symmetric part $\frac{1}{2}(M^T + M)$, restricted to mean-zero profiles $\tilde{\pi} \in \text{col}(P_c)$, is definite (positive definite for always-positive obstruction, negative definite for always-negative). This uses the identity $\tilde{\pi}^T A \tilde{\pi} = \tilde{\pi}^T \frac{A + A^T}{2} \tilde{\pi}$ for any matrix $A$.

Velocity ($M = I$) satisfies both conditions and is indeed obstructing: $I\mathbf{1} = \mathbf{1}$ and the symmetric part $I$ is positive definite. Acceleration ($M = D$, the first-difference matrix) fails condition (a): $D\mathbf{1} \notin \text{span}(\mathbf{1})$, so acceleration is non-obstructing and can be decorrelated from velocity. 

\textbf{Practical consequence.} Obstructing concepts are
essentially disguised versions of velocity. For the purpose of decorrelation, they can be seen as real signal and something for us to recover as part of the true velocity subspace, not something to decorrelate against.

\subsection{Achieving decorrelation in practice}
\label{app:decorrelation_practice}

\textbf{Setup.} Given $K$ non-obstructing temporal concepts
$\tau_1, \ldots, \tau_K$ to decorrelate pairwise, and a pool of $M$ candidate motion profiles $\pi_1, \ldots, \pi_M$, we
precompute for each profile $\pi_j$ and each pair $(k, k')$ the within-video covariance:
\begin{equation}
    c_{j,kk'} = \text{Cov}_t(\tau_k^{(\pi_j)},
    \tau_{k'}^{(\pi_j)})
\end{equation}
This gives a matrix $C \in \mathbb{R}^{M \times \binom{K}{2}}$ whose columns correspond to concept pairs and whose entries are known numbers determined by the profiles. We seek weights $w \in \mathbb{R}^M$ satisfying:
\begin{itemize}[leftmargin=*, itemsep=2pt]
    \item $w_j \geq 0$ for all $j$ and $\sum_j w_j = 1$ (the weights define a distribution over profiles)
    \item $C^T w = 0$ (all $\binom{K}{2}$ pairwise decorrelation conditions)
    \item $\sum_j w_j \text{Var}_t^{(\pi_j)}(\tau_k) > 0$ for each $k$ we seek to recover (the concept must actually vary, otherwise the test is vacuous for that concept)
\end{itemize}

\textbf{Linear programming.} This is a linear feasibility problem: $\binom{K}{2}$ decorrelation constraints plus one normalization constraint give $\binom{K}{2} + 1$ linear equations in $M$ unknowns, plus non-negativity. When $M \gg \binom{K}{2} + 1$, the system is underdetermined and generically has solutions.

Feasibility requires, for each pair $(k, k')$, that the values $\{c_{j,kk'}\}_{j=1}^M$ include both positive and negative entries. If the covariance for some pair has a single sign across all non-constant profiles, that pair involves an obstructing relationship and must be removed from the decorrelation list. Standard LP solvers determine feasibility and return weights in seconds, even for $M = 10{,}000$ profiles and $K = 50$ concepts. If a given set of motion profiles $(\pi_1,...,\pi_M)$ fail to give a satisfactory solution, it is cheap to resample all $M$ motion profiles, recalculate $C$ and try again.

\textbf{Constructing the stimulus set from weights.} Each video is parameterized by two independent quantities: the motion profile $\pi_j$ and the starting position $(x_0, y_0)$. Given weights $w_1, \ldots, w_M$, we allocate $n_j = \text{round}(N \cdot w_j)$ videos to profile $\pi_j$, so that $\sum_j n_j \approx N$. For each of these $n_j$ videos, $(x_0, y_0)$ is drawn independently and uniformly from the spawning range satisfying the coverage condition. Videos sharing a velocity profile differ in starting position and are therefore distinct.

This connects to the NFP test. The NFP property requires randomized starting positions but works with any assignment of velocity profiles (Theorem~\ref{thm:main} holds for each fixed profile). Decorrelation requires a specific distribution over velocity profiles but does not require randomized starting positions. A stimulus set achieving both purposes simultaneously needs both randomized starting positions and the LP weighted profile distribution.

\textbf{Application to this paper.} We use $K = 5$ temporal concepts (speed, $v_x$, $v_y$, $|a|$, direction), giving $\binom{5}{2} = 10$ pairwise decorrelation constraints. We generate a pool of $M = 4000$ candidate velocity profiles and solve the LP to obtain weights. The resulting stimulus set of $N = 3000$ videos has pairwise covariance $\leq 0.0022$ for every concept pair. The residual is due to the rounding of continuous weights to discrete video counts.

\subsection{Limitations of decorrelated recovery}
\label{app:decorrelation_limits}

Structured decorrelation (Section~\ref{app:structured_decorrelation}) recovers $\bar{c}_\tau \propto d_\tau$ under two assumptions: additive linear encoding with a known, correct list of temporal concepts. In this section we show that the latter assumption cannot be satisfied in practice, even if we assume the former, and that the consequences are severe: the direction $\bar{c}_\tau$ that maximally covaries with $\tau$ can be orthogonal to the true encoding direction $d_\tau$.

\textbf{Can we avoid needing a concept list?} Say we want to recover the true encoding subspace for velocity. A natural idea is universal decorrelation: find a single distribution $P$ over all possible movement profiles that simultaneously achieves $\mathbb{E}_P[\text{Cov}_t(\tau_k, \text{vel})] = 0$ for every non-obstructing concept $\tau_k$, while preserving the velocity signal $\mathbb{E}_P[\text{Var}_t(\text{vel})] > 0$. This would render $\bar{c}_{\text{vel}}$ free of contamination from any non-obstructing concept without requiring us to list them. We now show that no such distribution exists.

\textbf{Contamination-to-signal ratio.} For a distribution $P$ over velocity profiles with $\mathbb{E}_P[\text{Var}_t(\text{vel})] > 0$, define for each concept $\tau_k$:
\begin{equation}
    R_{\tau_k}(P) = \frac{|\mathbb{E}_P[\phi_{\tau_k}(\pi)]|}{\mathbb{E}_P[\phi_v(\pi)]}
\end{equation}
where $\phi_{\tau_k}(\pi) = \text{Cov}_t(\tau_k^{(\pi)}, \text{vel}^{(\pi)})$ and $\phi_v(\pi) = \text{Var}_t(\text{vel}^{(\pi)})$, following the notation from Section~\ref{app:decorrelation_obstructing}. This measures contamination from $\tau_k$ relative to the velocity signal. The worst case over all non-obstructing concepts is $R^*(P) = \sup_{\tau_k\text{ non-obstructing}} R_{\tau_k}(P)$. If $R^*(P) = 0$, then $P$ is a universal solvent.

\begin{theorem}[No universal solvent]
\label{thm:no_universal}
For $T \geq 3$ and any distribution $P$ over velocity profiles with $\mathbb{E}_P[\phi_v(\pi)] > 0$:
\begin{equation}
    R^*(P) \geq \frac{T-2}{T-1}
\end{equation}
In particular, no distribution achieves $R^*(P) = 0$.
\end{theorem}

\textit{Proof.} We exhibit a family of non-obstructing linear concepts, one of which must have a large contamination ratio regardless of $P$.

\textbf{Step 1: Reduction to second moments.} Consider linear
concepts $\tau_k(t) = \sum_s M_{ts}\, v(s)$ satisfying condition (a) from Section~\ref{app:decorrelation_obstructing}, i.e.\ $M\mathbf{1} \in \text{span}(\mathbf{1})$. For these concepts, the cross-term $\bar{v} \cdot \mathbf{1}^T M^T \tilde{\pi}$ vanishes for all velocity profiles, so:
\begin{equation}
    T \cdot \phi_{\tau_k}(\pi) = \tilde{\pi}^T M^T \tilde{\pi} = \tilde{\pi}^T \tfrac{M^T + M}{2}\, \tilde{\pi}
\end{equation}
where $\tilde{\pi} = P_c \pi$ is the mean-centered velocity profile. Since $P_c$ has rank $T-1$, we write $\tilde{\pi} = U\mathbf{a}$ where $U \in \mathbb{R}^{T \times (T-1)}$ has orthonormal columns spanning $\text{col}(P_c)$, and $\mathbf{a} \in \mathbb{R}^{T-1}$. Define $S'_M = \tfrac{1}{2} U^T(M^T + M)U$, a $(T{-}1) \times (T{-}1)$ symmetric matrix. Then:
\begin{equation}
    T \cdot \phi_{\tau_k}(\pi) = \mathbf{a}^T S'_M \mathbf{a}
\end{equation}
Taking the expectation over $P$ and using the cyclic property of trace:
\begin{equation}
    T \cdot \mathbb{E}_P[\phi_{\tau_k}] = \mathbb{E}_P[\mathbf{a}^T S'_M \mathbf{a}] = \text{tr}(S'_M \cdot \Sigma_{\mathbf{a}})
    \label{eq:trace_formula}
\end{equation}
where $\Sigma_{\mathbf{a}} = \mathbb{E}_P[\mathbf{a}\mathbf{a}^T]$ is the $(T{-}1) \times (T{-}1)$ positive semidefinite second-moment matrix. For velocity ($M = I$, so $S'_M = I_{T-1}$), the signal is $T \cdot \mathbb{E}_P[\phi_v] = \text{tr}(\Sigma_{\mathbf{a}})$.

\textbf{Step 2: Exhibiting non-obstructing concepts and the bound.} For each $i \in \{1,\ldots,T-1\}$, define the $(T{-}1)\times(T{-}1)$ symmetric matrix $S'_i = I_{T-1} - e_ie_i^T$. This has eigenvalue $0$ for $e_i$ and eigenvalue $1$ for every other basis vector, so it is positive semidefinite but not positive definite (since $T-1\geq 2$, there exists nonzero $\mathbf{a}$ with $\mathbf{a}^T S'_i\mathbf{a}=0$). Hence the corresponding concept is non-obstructing. 

Each $S'_i$ can be realized as an actual linear concept satisfying condition (a), via $M_i = US'_iU^T + \frac{c}{T}\mathbf{1}\mathbf{1}^T$ for any scalar $c$. To verify: $M_i\mathbf{1} = US'_iU^T\mathbf{1} + c\mathbf{1} = c\mathbf{1}$, confirming condition (a), and $U^TM_iU = U^TUS'_iU^TU + \frac{c}{T}U^T\mathbf{1}\mathbf{1}^TU = S'_i$, confirming the correct reduced form. Both use $U^TU = I_{T-1}$ and $U^T\mathbf{1} = \mathbf{0}$ (because $U$'s columns span the mean zero subspace).
Now, by the trace formula Equation~\ref{eq:trace_formula}, the contamination ratio for concept $S'_i$ is:
$$R_{S'_i}(P) = \frac{|\operatorname{tr}(S'_i\Sigma_\mathbf{a})|}{\operatorname{tr}(\Sigma_\mathbf{a})} = \frac{\operatorname{tr}(\Sigma_\mathbf{a}) - [\Sigma_\mathbf{a}]_{ii}}{\operatorname{tr}(\Sigma_\mathbf{a})} = 1 - \frac{[\Sigma_\mathbf{a}]_{ii}}{\operatorname{tr}(\Sigma_\mathbf{a})}$$
where we used $\operatorname{tr}(S'_i\Sigma_\mathbf{a}) = \operatorname{tr}(\Sigma_\mathbf{a}) - [\Sigma_\mathbf{a}]_{ii} \geq 0$ since $\Sigma_\mathbf{a}$ is PSD. The smallest diagonal entry satisfies $\min_i [\Sigma_\mathbf{a}]_{ii} \leq \frac{1}{T-1}\operatorname{tr}(\Sigma_\mathbf{a})$ (the minimum is at most the average). Let $i^*$ be the minimizing index. Then $\frac{[\Sigma_\mathbf{a}]_{i^*i^*}}{\text{tr}(\Sigma_\mathbf{a})}\leq\frac{1}{T-1}$, so:
$$R^*(P) \geq R_{S'_{i^*}}(P) \geq 1 - \frac{1}{T-1} = \frac{T-2}{T-1}$$
For $T = 8$ as in this paper, the floor is $\frac{6}{7} \approx 0.857$: even the best stimulus design leaves at least one non-obstructing concept with high contamination.

\textbf{We are therefore restricted to structured decorrelation}, and we need a finite list of concepts chosen in advance. We now examine what structured decorrelation actually recovers, and show that the result depends on assumptions that cannot be verified.

\textbf{With the true concept list and perfect decorrelation, recovery succeeds.} Suppose the model's representation truly decomposes as $h(V,t) = \sum_{k=1}^{K} \tau_k(V,t)\, d_k + \text{static}(V,t)$ for a specific list of concepts $\tau_1, \ldots, \tau_K$. With perfect decorrelation ($\mathbb{E}[\text{Cov}_t(\tau_j, \tau_k)] = 0$ for $j \neq k$), we get $\bar{c}_{\tau_k} \propto d_k$ for each $k$. The concept list determines what $d_\tau$ means, and perfect decorrelation recovers it.

\textbf{But a different concept list gives a different encoding direction.} In practice, we do not know which concepts the model encodes. Different decompositions of the same representation yield different encoding directions. Consider a model encoding velocity and some kind of acceleration:
$$h(V,t) = \text{vel}(V,t) \cdot d_v + \text{acc}(V,t)\cdot d_a + \text{static}(V,t)$$
For any $\alpha \in \mathbb{R}$, define $\tau_2' = \text{acc} - \alpha \cdot \text{vel}$. Then consider that for the same $h$:
\begin{equation}
    h(V,t) = \text{vel}(V,t) \cdot \underbrace{(d_v + \alpha\,
    d_a)}_{d_v'} + \tau_2'(V,t) \cdot d_a + \text{static}(V,t)
\end{equation}
Both decompositions produce identical representations for all inputs. If we decorrelate against $\tau_2'$ instead of $\text{acc}$, we recover $d_v' = d_v + \alpha\, d_a$, which differs from $d_v$ by an arbitrary multiple of $d_a$. Since $\alpha$ ranges over all of $\mathbb{R}$, there are infinitely many valid decompositions, each defining a different ``velocity encoding direction.'' So we see that the encoding direction for velocity is only defined relative to a choice of concept list.

\textbf{Without full decorrelation, $\bar{c}_{\text{vel}}$ can be orthogonal to $d_v$.} In practice, when we commit to a specific concept list, recovery fails when decorrelation is incomplete, which in practice it most likely always is, because the true concept list is privileged information that we don't have. We demonstrate this with a concrete construction.

Let $d = 2$, $K = 2$, with $d_v = (1, 0)^T$ and $d_a = (1, 1)^T$. The velocity subspace is $\mathcal{V} = \text{span}(d_v)$ (assuming no obstructing concepts besides velocity itself). We have:
\begin{equation}
    \bar{c}_{\text{vel}} = d_v \cdot \gamma_1 + d_a \cdot \gamma_2 = (\gamma_1 + \gamma_2,\; \gamma_2)^T
\end{equation}
where $\gamma_1 = \mathbb{E}_P[\text{Var}_t(\text{vel})] > 0$ and $\gamma_2 = \mathbb{E}_P[\text{Cov}_t(\text{acc}, \text{vel})]$. The projection of $\bar{c}_{\text{vel}}$ onto $\mathcal{V}$ has magnitude $\gamma_1 + \gamma_2$. Setting $\gamma_2 = -\gamma_1$ makes $\bar{c}_{\text{vel}} = (0, -\gamma_1)^T$, which is orthogonal to $d_v$ and to $\mathcal{V}$.

To show that $\gamma_2 = -\gamma_1$ is achievable, we can verify that stimulus sets exist where $\gamma_2 / \gamma_1 < -1$, but we omit doing so here because it is quite believable.

In the same thinking, say we decorrelate all known physics concepts from velocity, including displacement, acceleration, jerk, snap etc., then the model could still encode some convoluted $\tau_k$ such that for any stimulus set we use, we cannot ensure that a situation like $\gamma_2 \approx -\gamma_1$ doesn't occur that makes $\bar{c}_{\text{vel}}$ (near)-orthogonal to $d_v$ and to the velocity subspace $\mathcal{V}$.

\textbf{We summarize why it is not meaningful to ask if a feature encodes a concept $\tau_k$ in the NFP framework.} The argument proceeds through two independent obstacles:

(i) \emph{The concept list is not identifiable.} The encoding direction $d_\tau$ is only defined relative to a chosen decomposition of the representation. Different valid decompositions assign different $d_\tau$'s to the same model, and we cannot determine the correct decomposition from the model's behavior alone.

(ii) \emph{Full decorrelation is unachievable in practice.} Recovering $d_\tau$ requires decorrelating against every other concept in the correct decomposition. We don't have the list of all concepts the model encodes, and Theorem~\ref{thm:no_universal} proves that no stimulus set decorrelates all non-obstructing concepts simultaneously. Without full decorrelation, the max-covariance direction $\bar{c}_\tau$ can be orthogonal to $d_\tau$.

Thus, claiming a specific direction encodes a specific temporal concept $\tau$ is not a well-grounded statement within the NFP framework. The NFP test reliably identifies \emph{that} a direction is temporal, but not \emph{which} temporal concept it encodes. Structured decorrelation remains useful for constructing stimulus sets with controlled statistical properties (Section~\ref{app:decorrelation_practice}), but should not be interpreted as recovering true encoding directions.

\section{Model Specifications}
\label{app:model_specs}

Table~\ref{tab:model_specs} summarises the architecture and activation 
extraction details for each model used in this paper.

\begin{table}[h]
\centering
\small
\begin{tabular}{lcccccccr}
\toprule
Model & Arch. & Training & Frames & Res. & $\ell_p$ 
      & Grid & $d$ & Layers \\
\midrule
VideoMAE          & ViT-B & MAE pretrain 
                  & 16 & $224^2$ & 2 
                  & $8\!\times\!14\!\times\!14$ & 768 & 2, 5, 11 \\
VideoMAE-FT       & ViT-B & \,+ SSv2 finetuning
                  & 16 & $224^2$ & 2 
                  & $8\!\times\!14\!\times\!14$ & 768 & 2, 5, 11 \\
V-JEPA2           & ViT-L & self-sup.\ JE 
                  & 16 & $256^2$ & 2 
                  & $8\!\times\!16\!\times\!16$ & 1024 & 5, 11, 23 \\
TimeSformer       & ViT-B & + SSv2 finetuning
                  & 8  & $224^2$ & 1 
                  & $8\!\times\!14\!\times\!14$ & 768 & 2, 5, 11 \\
DINOv2$^\dagger$  & ViT-B & image self-sup. 
                  & 16 & $224^2$ & --- 
                  & $16\!\times\!16$ & 768 & final \\
\bottomrule
\end{tabular}
\medskip
\caption{Model specifications. \emph{Grid}: spatiotemporal token grid 
($T/\ell_p \times H_p \times W_p$). \emph{$d$}: representation 
dimension per token. \emph{Layers}: transformer layers from which 
post-MLP residual-stream activations are extracted. 
$^\dagger$DINOv2 processes each of the 16 frames independently 
with no cross-frame attention and serves as a negative control. TimeSformer's 8 frames are every second frame of the 16-frame clip, and NFP covariance for TimeSformer is computed over these 8 time steps.
Model identifiers are listed in Table~\ref{tab:checkpoints}.}
\label{tab:model_specs}
\end{table}

\begin{table}[h]
\centering
\small
\begin{tabular}{ll}
\toprule
Model & Identifier \\
\midrule
VideoMAE        & \texttt{MCG-NJU/videomae-base-ssv2} \\
VideoMAE-FT     & \texttt{MCG-NJU/videomae-base-finetuned-ssv2} \\
V-JEPA2         & \texttt{facebook/vjepa2-vitl-fpc16-256-ssv2} \\
TimeSformer  & \texttt{facebook/timesformer-base-finetuned-ssv2} \\
DINOv2          & \texttt{facebook/dinov2-base} \\
\bottomrule
\end{tabular}
\medskip
\caption{HuggingFace checkpoint identifiers for each model.}
\label{tab:checkpoints}
\end{table}

\section{SAE training details}
\label{app:training}

Every sparse autoencoder in the paper is a standard $\ell_1$ SAE trained with the same recipe and hyperparameters; only the model width (hence dictionary size), the training corpus, and the number of gradient steps differ. Each SAE encodes an activation $x$ as $f = \mathrm{ReLU}(W_\mathrm{enc}(x - b_\mathrm{dec}) + b_\mathrm{enc})$ and reconstructs $\hat{x} = W_\mathrm{dec} f + b_\mathrm{dec}$, with an $8\times$ expansion factor. All hyperparameters were fixed in advance, and no early stopping or checkpoint selection was used. Activations are stored in float16 and cast to float32 at load; training directly in float16 diverges to NaN. Table~\ref{tab:alive_on_ssv2} reports reconstruction quality and the number of features alive on SSv2.

\begin{table}[H]
\centering
\begin{tabular}{lc}
\toprule
Hyperparameter & Value \\
\midrule
Expansion factor & $8\times$ \\
$\ell_1$ penalty $\lambda$ & 0.1 \\
Dead-feature coefficient & 0.03 \\
AuxK $\alpha$ & 0.0 \\
Learning rate & $2\times10^{-4}$ \\
Batch size & 4096 \\
Gradient steps & 15{,}320 \\
\bottomrule
\end{tabular}
\caption{Hyperparameters of the standard $\ell_1$ SAEs, identical for every model: finetuned and pretrained VideoMAE, TimeSformer at layers 11 and 10, V-JEPA2, DINOv2, and the synthetic control. Only the dictionary size, corpus, and step count vary (Table~\ref{tab:sae_corpora}); the synthetic SAE uses 5{,}000 steps because its corpus is smaller.}
\label{tab:sae_params}
\end{table}

\begin{table}[H]
\centering
\small
\begin{tabular}{llccc}
\toprule
SAE & Layer & $d \to D$ & Tokens per clip & Training tokens \\
\midrule
VideoMAE finetuned  & 11 & $768 \to 6144$ & 1568 & ${\sim}6.3$M \\
VideoMAE            & 11 & $768 \to 6144$ & 1568 & ${\sim}6.3$M \\
TimeSformer         & 11 & $768 \to 6144$ & 1569 (incl.\ CLS) & ${\sim}6.3$M \\
TimeSformer         & 10 & $768 \to 6144$ & 1569 (incl.\ CLS) & ${\sim}6.3$M \\
V-JEPA2             & 23 & $1024 \to 8192$ & 2048 & ${\sim}8.2$M \\
DINOv2              & final & $768 \to 6144$ & 256 per frame & ${\sim}16$M \\
Synthetic           & -- & $768 \to 6144$ & 8 & 19{,}200 \\
\bottomrule
\end{tabular}
\caption{Per-SAE corpus and dictionary size. Every SAE is trained on the activations of 4000 SSv2 training clips from its own model, except the synthetic SAE, which uses an 80/20 train/val split of the 3000 generated videos. DINOv2 processes each frame independently as 256 spatial patch tokens.}
\label{tab:sae_corpora}
\end{table}

\begin{table}[H]
\centering
\small
\begin{tabular}{lccc}
\toprule
\textbf{SAE} & \textbf{FVE} & \textbf{$L_0$} & \textbf{Alive on SSv2} \\
\midrule
VideoMAE finetuned (L11) & 0.59 & 40.5 & 6067 / 6144 (98.7\%) \\
TimeSformer L11 & 0.90 & 12.8 & 5536 / 6144 (90.1\%) \\
TimeSformer L10 & 0.94 & 3.4 & 5577 / 6144 (90.8\%) \\
V-JEPA2 (L23) & 0.67 & 57.5 & 8183 / 8192 (99.9\%) \\
DINOv2 (final) & $\sim$0.80 & $\sim$28 & 6122 / 6144 (99.6\%) \\
Synthetic control & $\sim$0.51 & $\sim$71 & n/a \\
\bottomrule
\end{tabular}
\caption{FVE (fraction of variance explained) $=1-\text{FVU}$, on held-out activations. FVE and $L_0$ come from training logs, except for the DINOv2 and synthetic SAEs, which were not logged and are recomputed from their checkpoints. The recomputation gives FVE about 0.04 above the logged metric (0.63 vs.\ 0.59 on VideoMAE finetuned), so those two FVE values are marked approximate. Alive on SSv2 counts features that fire at least once on SSv2 validation. The synthetic SAE has no SSv2 count because it is trained on synthetic control activations.}
\label{tab:alive_on_ssv2}
\end{table}

\section{Valid features per dictionary}
\label{app:live}

Table~\ref{tab:live} reports how many features are valid, meaning a finite $t$ from a non-degenerate one-sample t-test, for the raw layer and the SAE dictionary of each model. Raw dictionaries are dense, so every dimension is active on the stimulus and all $D$ are valid. The SAE dictionaries are sparse and trained on natural SSv2 video, so most features never fire on the out-of-distribution ball stimulus and only a small fraction are valid. This valid set is the population the test actually evaluates, while the Bonferroni denominator stays the full dictionary size $D$. Note however, that the vast majority of the SAE features for all models are alive (fires at least once) on held-out SSv2 (Table~\ref{tab:alive_on_ssv2}). Thus in the interventions section Appendix~\ref{app:interventionanalysis} it is not unexpected to see span erasure having small effects, because the erasure is an erasure of a small number of features within a dictionary of many thousands of active features. 

We consider that a non-valid feature on the stimulus set has two possibilities. $C(V)=0$ for all videos is one, so the feature is dead on the stimulus, and shouldn't be called static or temporal. $C(V)=c\neq0$ for all videos is the other, which probably signals some kind of artifact where the NFP test has gone wrong, either regarding the stimulus or $N$ being too small or something else. So in practice, non-valid features shouldn't encode ball-motion related temporal concepts, although they may encode temporal concepts not represented by the stimulus set, e.g. specifically a hand moving left, or increasing brightness. 

In the code, we do not add an additional filter to remove near-dead features from the NFP test calculations on top of non-valid features. This means that under the right conditions the NFP test can be quite well behaved, as seen in Section~\ref{sec:synth_main} where Table~\ref{tab:projfrac} seems to show reasonable $t$-scores and not any overly large $t$-scores that we'd see from a misbehaving NFP test due to near dead features.

\begin{table}[h]
\centering
\begin{tabular}{lrr}
\toprule
Model & Axes (valid / $D$) & SAE (valid / $D$) \\
\midrule
VideoMAE finetuned & 768 / 768 & 637 / 6144 \\
VideoMAE & 768 / 768 & 571 / 6144 \\
TimeSformer L11 & 768 / 768 & 93 / 6144 \\
V-JEPA2 & 1024 / 1024 & 2783 / 8192 \\
DINOv2 & 768 / 768 & 597 / 6144 \\
\bottomrule
\end{tabular}
\caption{Valid features (finite $t$ in the one-sample t-test) for the axes and SAE dictionaries of each model. Every axes dimension is valid. The sparse SAE dictionaries have most features non-valid on ball stimulus, which is expected since the SAEs were trained on SSv2 and the synthetic ball stimulus is both out of distribution and involves extremely simple videos. SAE counts are on the decorrelated stimulus.}
\label{tab:live}
\end{table}

\section{Identifying temporal features: results across layers}
\label{app:main_layers}

\begin{table}[H]
\centering
\small
\begin{tabular}{lllrrrrrrr}
\toprule
Model & $D$ & Layer & speed & $v_x$ & $v_y$ & $|a|$ & direction & any & \% \\
\midrule
\multirow{3}{*}{VideoMAE finetuned} & \multirow{3}{*}{768} 
  & 2 & 1 & 2 & 21 & 499 & 15 & 507 & 66.0 \\
 & & 5 & 258 & 244 & 144 & 522 & 52 & 672 & 87.5 \\
 & & 11 (final) & 352 & 359 & 151 & 496 & 89 & 707 & 92.1 \\
\midrule
\multirow{3}{*}{VideoMAE} & \multirow{3}{*}{768} 
  & 2 & 3 & 3 & 27 & 506 & 27 & 521 & 67.8 \\
 & & 5 & 8 & 24 & 33 & 559 & 32 & 576 & 75.0 \\
 & & 11 (final) & 6 & 23 & 58 & 515 & 29 & 540 & 70.3 \\
\midrule
\multirow{3}{*}{V-JEPA2} & \multirow{3}{*}{1024} 
  & 5 & 214 & 40 & 29 & 351 & 2 & 525 & 51.3 \\
 & & 11 & 327 & 277 & 124 & 616 & 21 & 856 & 83.6 \\
 & & 23 (final) & 45 & 2 & 4 & 254 & 0 & 294 & 28.7 \\
\midrule
\multirow{3}{*}{TimeSformer} & \multirow{3}{*}{768} 
  & 2 & 4 & 0 & 87 & 438 & 0 & 440 & 57.3 \\
 & & 5 & 15 & 7 & 41 & 351 & 1 & 361 & 47.0 \\
 & & 11 (final) & 23 & 20 & 63 & 408 & 4 & 426 & 55.5 \\
\midrule
DINOv2 & 768 & final & 0 & 0 & 0 & 0 & 0 & 0 & 0.0 \\
\bottomrule
\end{tabular}
\caption{Representation space dimension axes flagged by the NFP test at a given layer at the Bonferroni threshold $p < 0.05/D$. $D$ means the number of directions $w$ tested, which here is equivalent to the number of dimensions. Per-profile columns count axes significant for that $\tau_k$, and an axis can be significant for several. ``any'' counts dimensions significant for at least one profile.}
\label{tab:axes_across_layers}
\end{table}

Comparing between different layers of the same model, we see that VideoMAE finetuned develops temporal axes features in the early and middle layers and the counts of these temporal axes features persist till the final layer. However, for V-JEPA2, the count decreases significantly at the final layer. We speculate this might mean that V-JEPA2 develops a more structured representation of temporal concepts in its later layers, which VideoMAE finetuned never does. At the very least, these results are highly suggestive that the two models learn temporal information differently.

\section{Synthetic control: construction and supplementary results}
\label{app:synth_details}

\subsection{Full construction details}
\label{app:synth_construction}

The synthetic representation $h(V,t) \in \mathbb{R}^{768}$ is  constructed to match the dimensionality of VideoMAE and TimeSformer. The three projection matrices $W_\tau \in \mathbb{R}^{5\times768}$, $W_\sigma \in \mathbb{R}^{100\times768}$, and $W_p \in \mathbb{R}^{50\times768}$ are obtained by drawing a $768 \times 155$ matrix with i.i.d.\ $\mathcal{N}(0,1)$ entries, and computing its QR decomposition. This guarantees mutual orthonormality: every row is a  unit vector and all 155 rows are mutually orthogonal.

\textbf{Temporal component.} 
$\tau_{\mathrm{norm}}(V,t) \in \mathbb{R}^5$ stacks the five  temporal concept profiles (speed, $v_x$, $v_y$, $|a|$, direction), each z-scored to mean zero and unit variance using the global statistics across all $N = 3000$ videos and $T = 8$ time steps. Z-scoring ensures that each concept contributes equal variance to the temporal subspace, so that a random direction in $\operatorname{span}(W_\tau)$ samples all five concepts comparably. The profiles are computed from the same ball trajectories used in the stimulus set.

\textbf{Static component.}
$c(V) \in \mathbb{R}^{100}$ with each entry drawn $\sim\mathcal{N}(0,1)$ once per video and held constant across all $T$ steps. This represents per-clip content (object identity, background) constant across $t$. 

\textbf{Position-dependent component.}
Each of the 50 channels $m$ produces a scalar $g_m(V, p) = s_m(V) \cdot \exp\!\bigl(-\|p - \mu_m(V)\|^2 / 2\omega^2\bigr)$, where $s_m(V) \sim \mathrm{Uniform}\{+1,-1\}$ is a random sign, $\mu_m(V)$ is a random center, and both are drawn independently for each channel~$m$ and each video~$V$. The bump is evaluated at the ball's position $p = \mathrm{pos}(V,t)$ at each time step. The random-sign mechanism ensures that the expected within-video covariance with any temporal concept is exactly zero (see \S\ref{app:pos_sign}).

\textbf{Noise.}
$\varepsilon(V,t) \sim \mathcal{N}(0, 0.05^2\, I_{768})$ is i.i.d.\ across videos, time steps, and coordinates.  Without noise, $h(V,t)$ would lie exactly in a 155-dimensional subspace of $\mathbb{R}^{768}$.  An SAE trained on such exactly low-rank input  could achieve perfect reconstruction using only 155 decoder  directions, removing the sparsity pressure, making the SAE validation degenerate. The noise fills all 768 dimensions, forcing the SAE to learn a genuine sparse decomposition.  The scale $0.05$ is small relative to the planted components and is not tuned.

\subsection{Random-sign mechanism for position-dependent content}
\label{app:pos_sign}

The position-dependent component is designed so that each channel's expected within-video covariance with any temporal concept $\tau_k$ is exactly zero.

Fix a video $V$, a position channel $m$, and a temporal concept $k$.  The within-video covariance is: $$\operatorname{Cov}_t\!\bigl(g_m(V,\cdot),\; \tau_k(V,\cdot)\bigr) = \frac{1}{T}\sum_{t=1}^{T}\bigl(g_m(V,p_t) - \bar{g}_m\bigr) \bigl(\tau_k(V,t) - \bar{\tau}_k\bigr)$$
where $p_t = \mathrm{pos}(V,t)$. Since $g_m(V, p) = s_m(V) \cdot \exp(-\|p - \mu_m(V)\|^2/2\omega^2)$, the sign $s_m(V)$ factors out of every term:
$$\operatorname{Cov}_t(g_m, \tau_k) = s_m(V) \cdot \operatorname{Cov}_t\!\bigl(\exp(-\|p_\cdot - \mu_m\|^2/2\omega^2),\; \tau_k\bigr)$$ 
The second factor is a deterministic function of the trajectory and the bump center $\mu_m(V)$; call it $A(V, m)$.  Because $s_m(V) \sim \mathrm{Uniform}\{+1,-1\}$ independently of the trajectory and center:
$$\mathbb{E}_{s_m}\!\bigl[\operatorname{Cov}_t(g_m, \tau_k)\bigr] = \mathbb{E}_{s_m}[s_m] \cdot A(V,m) = 0 \cdot A(V,m) = 0$$
This holds for every video individually, regardless of the trajectory, the starting position, or the bump center.  Averaging over videos, $\mathbb{E}_V[\operatorname{Cov}_t(g_m, \tau_k)] = 0$, so the $t$-test's null hypothesis is satisfied.

Note that without the NFP theorem's compact-support and 
coverage-condition machinery (which applies to real model 
activations but not to this synthetic control, where no off-screen 
zeroing is enforced), such a fixed function $g$ would produce a 
nonzero mean covariance across videos. The random-sign 
construction sidesteps this by ensuring cancellation at the level 
of individual videos rather than relying on averaging over starting 
positions.

\subsection{Full $\beta^2$ sweep results}
\label{app:projfrac_full}

Table~\ref{tab:projfrac_full} reports the complete 26-level sweep. The main text (Table~\ref{tab:projfrac}) shows a representative subset. At $\beta^2 = 0$, we test 200 directions: 100 random unit vectors in the full space, 50 lying entirely in $\operatorname{span}(W_p)$, and 50 lying entirely in $\operatorname{span}(W_\sigma)$. In total, 850 directions are tested. We run the NFP test with Bonferroni correction at $\alpha/850$, giving a threshold of $|t| > 4.0$.

\begin{table}[h]
\centering
\small
\begin{tabular}{lccrrrrr}
\toprule
& & & \multicolumn{5}{c}{$\max_k |t|$, single-concept directions} \\
\cmidrule(lr){4-8}
$\beta^2$ & Flagged & Median $\max_k |t|$ & speed & $v_x$ & $v_y$ & $|a|$ & direction \\
\midrule
0 (random) & 0/100 & 1.5 & -- & -- & -- & -- & -- \\
0 (position only) & 0/50 & 1.5 & -- & -- & -- & -- & -- \\
0 (constant only) & 0/50 & 1.4 & -- & -- & -- & -- & -- \\
\midrule
$10^{-5}$ & 2/20 & 3.3 & 7.2 & 6.3 & 4.5 & 2.6 & 3.5 \\
$1.6\times10^{-5}$ & 6/20 & 3.7 & 7.6 & 7.8 & 1.5 & 1.2 & 5.2 \\
$2.5\times10^{-5}$ & 17/20 & 5.3 & 6.7 & 9.7 & 7.0 & 3.8 & 3.9 \\
$4.0\times10^{-5}$ & 20/20 & 6.0 & 10.1 & 10.0 & 7.0 & 5.1 & 6.8 \\
$6.3\times10^{-5}$ & 20/20 & 7.6 & 12.8 & 11.8 & 9.9 & 5.3 & 8.6 \\
$10^{-4}$ & 20/20 & 9.6 & 16.4 & 14.2 & 8.0 & 7.6 & 8.6 \\
$1.6\times10^{-4}$ & 20/20 & 11.4 & 19.5 & 17.5 & 10.8 & 8.1 & 11.5 \\
$2.5\times10^{-4}$ & 20/20 & 13.7 & 22.1 & 17.6 & 14.8 & 8.7 & 14.7 \\
$4.0\times10^{-4}$ & 20/20 & 14.9 & 25.9 & 19.0 & 19.4 & 10.4 & 14.7 \\
$6.3\times10^{-4}$ & 20/20 & 18.4 & 29.1 & 21.4 & 23.1 & 10.9 & 16.9 \\
$10^{-3}$ & 20/20 & 20.6 & 31.7 & 22.8 & 26.1 & 11.3 & 18.5 \\
$1.6\times10^{-3}$ & 20/20 & 22.1 & 33.3 & 22.9 & 28.2 & 12.0 & 19.2 \\
$2.5\times10^{-3}$ & 20/20 & 23.8 & 35.3 & 23.1 & 33.9 & 12.2 & 20.7 \\
$4.0\times10^{-3}$ & 20/20 & 26.3 & 37.1 & 23.5 & 36.8 & 12.5 & 21.2 \\
$6.3\times10^{-3}$ & 20/20 & 25.2 & 37.6 & 23.7 & 38.8 & 12.6 & 21.5 \\
$10^{-2}$ & 20/20 & 28.1 & 38.4 & 23.7 & 40.4 & 12.6 & 21.7 \\
$1.6\times10^{-2}$ & 20/20 & 30.2 & 39.0 & 23.8 & 42.4 & 12.5 & 21.9 \\
$2.5\times10^{-2}$ & 20/20 & 31.8 & 38.9 & 23.9 & 42.3 & 12.6 & 21.8 \\
$4.0\times10^{-2}$ & 20/20 & 30.1 & 39.4 & 23.9 & 43.1 & 12.7 & 22.0 \\
$6.3\times10^{-2}$ & 20/20 & 33.2 & 39.4 & 24.0 & 43.4 & 12.7 & 21.9 \\
$10^{-1}$ & 20/20 & 31.7 & 39.5 & 24.0 & 43.6 & 12.6 & 22.0 \\
$1.6\times10^{-1}$ & 20/20 & 33.6 & 39.5 & 24.0 & 43.9 & 12.7 & 22.0 \\
$2.5\times10^{-1}$ & 20/20 & 28.7 & 39.4 & 24.1 & 43.9 & 12.7 & 22.1 \\
$4.0\times10^{-1}$ & 20/20 & 29.5 & 39.5 & 24.0 & 44.1 & 12.7 & 22.0 \\
$6.3\times10^{-1}$ & 20/20 & 33.4 & 39.5 & 24.0 & 44.1 & 12.7 & 22.1 \\
$10^{0}$ & 20/20 & 31.3 & 39.5 & 24.0 & 44.1 & 12.7 & 22.0 \\
\bottomrule
\end{tabular}
\caption{Full projection-fraction sweep across all 26 log-spaced $\beta^2$ levels plus the three $\beta^2=0$ null sets. This is the complete version of Table~\ref{tab:projfrac}, which lists a representative subset. ``Flagged'' counts the random directions passing the Bonferroni threshold, and the last five columns give the $\max_k|t|$ of the single-concept directions.}
\label{tab:projfrac_full}
\end{table}

\subsection{Individual flagged SAE features}
\label{app:synth_features}

Table~\ref{tab:synth_features_app} lists the 8 SAE features 
flagged by the NFP test on the synthetic representation, with 
their ProjFrac and per-concept cosine similarities.

\begin{table}[h]
\centering
\small
\setlength{\tabcolsep}{4pt}
\begin{tabular}{llcccccc}
\toprule
& & & \multicolumn{5}{c}{$|\cos|$ with planted concept direction} \\
\cmidrule(lr){4-8}
Feature & Flagged for & ProjFrac $W_\tau$ 
        & speed & $v_x$ & $v_y$ & $|a|$ & direction \\
\midrule
feat00601 & speed, $v_x$, $v_y$ & 0.714 
          & 0.026 & 0.292 & \textbf{0.763} & 0.203 & 0.069 \\
feat03882 & speed, $v_x$, $v_y$ & 0.659 
          & 0.080 & 0.290 & \textbf{0.731} & 0.138 & 0.120 \\
feat00627 & speed, $v_x$, $v_y$ & 0.648 
          & 0.028 & 0.477 & \textbf{0.620} & 0.157 & 0.103 \\
feat03197 & all five & 0.608 
          & 0.121 & \textbf{0.610} & 0.174 & 0.027 & 0.437 \\
feat01793 & $v_x$, $v_y$, direction & 0.602 
          & 0.102 & \textbf{0.581} & 0.156 & 0.066 & 0.474 \\
feat01248 & speed, $v_x$, $v_y$ & 0.566 
          & 0.159 & 0.460 & \textbf{0.565} & 0.092 & 0.016 \\
feat00987 & $v_x$ & 0.456 
          & 0.106 & \textbf{0.623} & 0.040 & 0.233 & 0.023 \\
feat05391 & speed, $|a|$ & 0.407 
          & 0.104 & 0.013 & 0.020 & \textbf{0.628} & 0.029 \\
\bottomrule
\end{tabular}
\caption{The 8 flagged synthetic-SAE features.  Every flagged 
feature places $0.41$--$0.71$ of its norm in the planted 5-d 
temporal subspace.  The absolute cosine with each planted concept 
direction is listed, with the dominant direction in bold.}
\label{tab:synth_features_app}
\end{table}

All 8 features align strongly with one or more single planted concept 
directions (maximum absolute cosine $0.57$--$0.76$), but none reach $1$, which is expected due to SAE reconstruction being non-exact. Each feature seems to mix a dominant concept with smaller components of other planted directions and of non-temporal contents. This is a known property of SAE feature learning.

\section{Theory behind $\beta^2$ sweep results}
\label{app:saturation}

We try to make mathematically clear in this Appendix how the $t$ statistics in Section~\ref{sec:synth_main} are derived, what intuitively the saturation of the five single-concept curves in Table~\ref{tab:projfrac} mean, and why the saturation levels are determined by the stimulus set rather than by any property of the test.

\subsection{General case}

Fix a unit test direction $d = \beta\, u_\tau + \sqrt{1-\beta^2}\, u_\perp$ and a temporal concept profile $\tau_{k}$. We want to understand the NFP test's output: the $t$-statistic of the within-video covariances $\{C^{(k)}(V_1), \ldots, C^{(k)}(V_N)\}$.

Step 1: Activation. The activation of direction $d$ on the synthetic representation is:
$$\psi(V,t) = \langle h(V,t),\, d \rangle$$
Because the four components of $h$ (temporal, static, position-dependent, noise) live in mutually orthogonal subspaces, and $d$ decomposes into $u_\tau$ (in the temporal subspace) and $u_\perp$ (in its complement), the activation separates as:
$$\psi(V,t) = \beta \underbrace{\langle h(V,t),\, u_\tau \rangle}_{\text{temporal part}} + \sqrt{1-\beta^2} \underbrace{\langle h(V,t),\, u_\perp \rangle}_{\text{non-temporal part}}$$
Step 2: Within-video covariance. The covariance with concept $k$ is:
$$C^{(k)}(V) = \operatorname{Cov}_t\!\bigl(\psi(V,\cdot),\; \tau_{k}(V,\cdot)\bigr)$$
By linearity of covariance:
$$C^{(k)}(V) = \beta \cdot \operatorname{Cov}_t\!\bigl(\langle h, u_\tau \rangle,\; \tau_{k}\bigr) + \sqrt{1-\beta^2} \cdot \operatorname{Cov}_t\!\bigl(\langle h, u_\perp \rangle,\; \tau_{k}\bigr)$$

Step 3: Expected value across videos. The temporal part has a nonzero expected covariance (it carries genuine temporal signal), while the non-temporal part has zero expected covariance: the static component is constant within each video (zero within-video covariance with anything), the position-dependent component has zero expected covariance by the random-sign mechanism (Appendix~\ref{app:pos_sign}), and the noise is independent. Therefore:
$$\mathbb{E}_V\bigl[C^{(k)}(V)\bigr] = \beta \cdot \mathbb{E}_V\!\bigl[\operatorname{Cov}_t(\langle h, u_\tau \rangle, \tau_{k})\bigr]$$
This scales linearly with $\beta$.

Step 4: The $t$-statistic. The one-sample $t$-statistic under the null hypothesis $\mu_0 = 0$ is:
$$t^{(k)} = \frac{\bar{C}^{(k)}}{\operatorname{sd}(C^{(k)})/\sqrt{N}}$$
where $\bar{C}^{(k)}$ and $\operatorname{sd}(C^{(k)})$ are the sample mean and standard deviation of $C^{(k)}(V)$ over $N$ videos. The numerator scales as $\beta$. The denominator depends on the variance of $C^{(k)}(V)$ across videos, which has contributions from both the temporal part (variance $\propto \beta^2$) and the non-temporal part (variance $\propto 1 - \beta^2$). 

For a random $u_\tau$ in the 5-d temporal subspace, the expected covariance involves the within-video cross-covariances between the five planted concepts and the test concept $k$, weighted by the random loadings of $u_\tau$. Because these loadings vary across random draws, so does the test response, which is why we write the ``\emph{Median} $\max_k|t|$'' column of Table~\ref{tab:projfrac}. For closed-form expressions, we now specialize to the single-concept case.

\subsection{Single-concept case}

Set $u_\tau = [W_\tau]_k$, the $k$-th row of the temporal projection matrix, so that all temporal content concentrates on concept $k$. Recall, $$h(V,t) = \sum_{j=1}^{5} \tau_{j,\text{norm}}(V,t) [W_\tau]_j + \sum_{l=1}^{100} c_l(V) [W_\sigma]_l + \sum_{m=1}^{50} g_m(V,\text{pos})[W_p]_m + \varepsilon(V,t)$$ 
So taking the inner product with $[W_\tau]_k$, the temporal part of the activation becomes:
$$\beta \langle h(V,t),\, [W_\tau]_k \rangle = \beta\, \tau_{k,\text{norm}}(V,t) + \beta \langle \varepsilon(V,t),\, [W_\tau]_k \rangle$$
where $\tau_{k,\text{norm}} = (\tau_k - \mu_k)/\sigma_k$ is the globally z-scored concept profile, and all other components vanish by orthonormality. The noise term $\langle \varepsilon, [W_\tau]_k \rangle \sim \mathcal{N}(0, 0.05^2)$ is small. We drop it for clarity:
$$\psi(V,t) \approx \frac{\beta}{\sigma_k}\bigl(\tau_k(V,t) - \mu_k\bigr) + \sqrt{1-\beta^2} \langle h(V,t),\, u_\perp \rangle$$
Then since the global mean $\mu_k$ cancels in the within-video covariance (it shifts both $\psi$ and $\bar\psi$ equally), we get within-video covariance for concept $k$:
$$\operatorname{Cov}_t\!\bigl(\psi, \tau_k\bigr) = \frac{\beta}{\sigma_k} \operatorname{Cov}_t(\tau_k, \tau_k) + \sqrt{1-\beta^2} \operatorname{Cov}_t\!\bigl(\langle h, u_\perp \rangle, \tau_k\bigr)$$
$$= \frac{\beta}{\sigma_k} \operatorname{Var}_t(\tau_k(V,\cdot)) + \text{zero-mean term}$$
Then the sample mean across $N$ is:
$$\bar{C}^{(k)} \approx \frac{\beta}{\sigma_k} \overline{\operatorname{Var}_t(\tau_k)}$$
Then to derive the standard deviation across videos $sd(C^{(k)})$, consider that generally $\mathrm{Var}(A+B)=\mathrm{Var}(A)+\mathrm{Var}(B)+2\mathrm{Cov}(A,B)$, so if the last term is 0 we have:
$$\operatorname{sd}(C^{(k)}) \approx \sqrt{\frac{\beta^2}{\sigma_k^2}\, \sigma_{\text{signal}}^2 + (1-\beta^2)\, \sigma_{\text{non-temporal}}^2}$$
where $\sigma_{\text{signal}}^2 = \operatorname{Var}_V[\operatorname{Var}_t(\tau_k)]$ is the across-video variance of the within-video variance, and $\sigma_{\text{non-temporal}}^2$ is the variance contribution from the non-temporal component from $u_\perp$. To see why the last term is 0, we let $\eta(V)=\mathrm{Cov}_t\bigl(\langle h, u_\perp \rangle, \tau_k\bigr)$, then because by construction, $\mathbb{E}[\eta\mid\pi]=0$, where $\pi$ is the velocity profile, we get:
$$\mathbb{E}_V[\mathrm{Var}_t(\tau_k)\cdot\eta(V)]=\mathbb{E}_\pi[\mathrm{Var}_t(\tau_k)\cdot\mathbb{E}[\eta|\pi]]=0$$
$$\mathrm{Cov}_V(\mathrm{Var}_t(\tau_k), \eta(V))=\mathbb{E}_V[\mathrm{Var}_t(\tau_k)\cdot\eta(V)]-\mathbb{E}_V[\mathrm{Var}_t(\tau_k)]\cdot\mathbb{E}_V[\eta(V)]=0$$
So the cross term vanishes exactly.

So the $t$-statistic becomes:
$$|t^{(k)}| \approx \sqrt{N} \cdot \frac{(\beta/\sigma_k)\, \overline{\operatorname{Var}_t(\tau_k)}}{\sqrt{(\beta^2/\sigma_k^2)\, \sigma_{\text{signal}}^2 + (1-\beta^2)\, \sigma_{\text{non-temporal}}^2}}$$

Now we can understand why $|t|$ grows with $\beta$ and then saturates. When $\beta$ is small, the denominator is approximately $\sigma_{\text{non-temporal}}$, independent of $\beta$, so the $t$-statistic grows linearly with $\beta$:
$$|t^{(k)}| \approx \sqrt{N} \cdot \frac{\beta\, \overline{\operatorname{Var}_t(\tau_k)}}{\sigma_k\, \sigma_{\text{non-temporal}}} \;\propto\; \beta$$

When $\beta$ is large ($\beta \to 1$), $\sigma_k$ and $\beta$ cancel from numerator and denominator, then:
$$|t^{(k)}| \;\to\; t_{\text{sat}}^{(k)} = \sqrt{N} \cdot \frac{\overline{\operatorname{Var}_t(\tau_k)}}{\operatorname{sd}_V[\operatorname{Var}_t(\tau_k)]}$$
This is a constant determined entirely by the stimulus set, and serves as a sort of upper bound. This analysis roughly explains how the $t$-scores behave in Table~\ref{tab:projfrac} as $\beta^2$ increases. 

\subsection{Implications for interpreting real-model results}

The saturation value $t_{\text{sat}}^{(k)}$ expression is a fraction that suggests in this synthetic setup, the NFP test takes the within-video variance as signal, and across-video variance reduces the $t$ score. We can consider this as a hypothetical rule of thumb to understand statistically how the NFP test sensitivity and feature flagging works. From Table~\ref{tab:projfrac} we can also see that it is not unlikely this rule extends to other values of $\beta$ because acceleration $|a|$ consistently has a lower single concept direction $t$-score across all values of $\beta>0$. If we extrapolate this rule of thumb to the NFP test on actual models, then we can say that the NFP test in Section~\ref{sec:nfp} was conducted on a stimulus set that actually discouraged the flagging of features under acceleration $|a|$, and so the large concentration of features flagged under acceleration is likely not a test artifact, rather, video models genuinely produce features that covary more significantly with ground truth acceleration $|a|$ for some reason. One possible explanation is that video models dedicate more of their representation to acceleration-related content than to other motion concepts due to yet unknown explanations involving mechanisms from model training. Overall, this synthetic control sweep over $\beta$ provides evidence that validates the main results in Section~\ref{sec:nfp}, despite these main results being admittedly unexpected.

\section{Synthetic control with non-orthonormal temporal directions}
\label{app:synth_oblique}

The main synthetic control plants five mutually orthonormal temporal directions.  Real concept directions need not be orthogonal: speed and $v_x$ may share variance in a real model's representation.  To test robustness, we repeat the control with non-orthogonal temporal encoding directions.

We modify the temporal projection matrix by adding a shared component: $[W_\tau]_i \gets \sqrt{1-c}\,[W_\tau]_i + \sqrt{c}\,q$, where $q$ is a fixed random unit vector orthogonal to the temporal subspace. This gives every pair of temporal directions pairwise cosine similarity exactly $c$ while preserving unit norms. We test $c = 0.25$ and $0.5$, train a new SAE on each modified representation, and run the NFP test unchanged.

\begin{table}[h]
\centering
\small
\begin{tabular}{lccccc}
\toprule
Geometry & Flagged & PF (flagged) & PF (unflagged) 
         & MaxCos range & PF $< 0.05$ \\
\midrule
orthonormal ($c=0$) & 8 & 0.582 & 0.003 & 0.57--0.76 & 0 \\
pairwise $c=0.25$   & 17 & 0.351 & 0.004 & 0.33--0.59 & 0 \\
pairwise $c=0.5$    & 20 & 0.237 & 0.005 & 0.28--0.55 & 0 \\
\bottomrule
\end{tabular}
\caption{NFP results under correlated temporal directions. PF is the mean ProjFrac of flagged/unflagged features, computed against an orthonormal basis of the temporal subspace's span. The last column counts flagged features with $\mathrm{ProjFrac} < 0.05$ (false positives).}
\label{tab:synth_oblique}
\end{table}

This correlated setup flags more SAE features (17 at $c = 0.25$, 20 at $c = 0.5$). Every flagged feature at every $c$ has substantial ProjFrac into the temporal subspace, and no false positive appears (last column). Individual features align less cleanly with single planted directions as $c$ grows, which is expected. When planted directions overlap, a single SAE feature can reconstruct several concepts at once.

\section{Intervention Analysis additional details and experiments}
\label{app:interventionanalysis}

In this Appendix we provide a further set of experiments to substantiate our claim that we are unable to characterize temporal features in video models as more important for temporally demanding tasks compared to non-temporal features. All experiments are done using the SSv2 dataset, and involve variants of SSv2 classification tasks.

The three models we used all contain slightly different setups for the classification heads we use. The finetuned VideoMAE model comes with a head from the finetuning procedure. The V-JEPA2 model comes with an attentive probe trained on top of the encoder for the 174 SSv2 classes. The TimeSformer model only uses the CLS token from its last layer in downstream classification. The CLS token is transformed through multiplication with a matrix to get a 174-dimension vector which represents the classification readouts. As a result when we intervene on TimeSformer, we need to intervene on layer 10, one layer before the CLS token, and not layer 11, where the tokens besides the CLS token never affect its downstream performance. 

Feature amplification and flipper experiments in this Appendix re-add the SAE reconstruction error, so a pass through the intervention pipeline with no perturbation returns the original model's prediction exactly. The baseline rows of every table in this section are therefore also the no-perturbation result of the SAE pipeline. We do this because SAE reconstruction performance differs across models, re-adding the error term keeps the intervention comparisons free of this reconstruction loss. Random controls are selected from \emph{all} non-flagged features with 3 seeds. This includes valid non-flagged features, and non-valid features (Appendix~\ref{app:live}). Low Static controls are drawn from features whose covariance with every temporal concept gives $|t| < 2$ (compared to the temporal threshold of $|t|>4+$). We note that splitting into the following categories may be more suitable here for analysis on some sections: \{temporal, valid \& non-flagged i.e. static (split into intermediate and low static if helpful), non-valid\}.

While non-valid features in practice shouldn't encode ball-motion related temporal concepts, they may encode temporal concepts not represented by the stimulus set, e.g. specifically a hand moving left, or increasing brightness. Likewise, valid non-flagged features, including low-static features, are not necessarily static across all stimuli for the same reason. That being said, because SSv2 is motion-heavy (even as ball-related motion may not fully encapsulate motion concepts in SSv2), we believe it is not unreasonable for non-valid features and valid non-flagged features to be taken (with caution) as static or near-static across all relevant stimuli for the sake of analysis.

\subsection{Feature amplification}

We first perform amplification on the SAE feature activation value $f_i \to \alpha f_i$ with $\alpha > 1$ for all features in the temporal/low static/random set. We evaluate each model on all SSv2 Classes, using 10 random videos for each class for a total of $n=1740$ videos.

\begin{table}[h]
\centering
\small
\begin{tabular}{@{}lccccc@{}}
\toprule
Model & Feature Count & Baseline & Identified & Low Static & Random\\
\midrule
VideoMAE finetuned & 109 & 0.629 & 0.634 & 0.631 & 0.627\\
V-JEPA2            &  62 & 0.681 & 0.678 & 0.680 & 0.675\\
TimeSformer L10 & 16 & 0.495 & 0.494 & 0.489 & 0.495\\
\bottomrule
\end{tabular}
\caption{Performance at amplification $\alpha = 3$. Columns show all-class accuracy. TimeSformer SAE features are drawn from layer 10, not layer 11.}
\label{tab:repair}
\end{table}

Although we only report $\alpha=3$ here, these results are representative of the changes observed for $\alpha \in (2, 10]$, with any improvement observed usually peaking between $\alpha=3$ and $\alpha=5$ and dropping after that. Here we again see that amplification behavior differs across models. VideoMAE finetuned produces what we'd most expect from amplifying temporal features: a modest performance increase over both the baseline and amplifying non-temporal features. The opposite is true for V-JEPA2, with amplifications harming performance, and for TimeSformer amplifications seem to have no effect. We acknowledge that these changes compared to the baseline are likely within noise, and suggest taking the results here in conjunction with results in other parts of the Appendix.

\subsection{Span erasure effect on binary classification}

We consider that if temporal features truly did matter for temporally demanding tasks more so than non-temporal features, then conducting span erasure of temporal features should more negatively impact performance on binary classification of temporally opposing pairs like pushing left-to-right vs right-to-left. Such a test is one of the most simple tests and a negative result on this test would be highly suggestive that temporal features don't matter as much as we think for SSv2 classification. In general here, we follow the original span erasure procedure in Section~\ref{sec:steer}. For a pair (A, B), we take the pre-softmax logits, keep only the logits for A and B, softmax over just those two, and predict the larger. Binary accuracy is the two-way correct classification rate over 10 validation clips per class (20 per pair).

In particular, we investigate the 35 kinematic pairs we could identify in SSv2. They are distinguished by a single kinematic variable such as:

1. Direction: e.g. pushing left-to-right vs right-to-left, camera up vs down.\\
2. Speed: e.g. Falling like a feather vs a rock.\\
3. Motion outcome: e.g. spinning continues vs stops, pushing almost falls off vs falls off, lifting without a drop vs then dropping.

We believe that kinematic pairs are most likely to exhibit the behavior we expect because the SAE features identified were identified using a ball-motion stimulus set, meaning that these flagged SAE features are encoding some motion related concepts, and not necessarily other kinds of temporal concepts. We let $\Delta=\text{erased} - \text{baseline}$, and define $\Delta_{temporal}$ as the change from erasing the span of identified features, and $\Delta_{rnd}$ as the mean change over the three random seeds.

\begin{table}[t]
\centering
\small
\begin{tabular}{@{}lccccc@{}}
\toprule
Model & Feature Count & Baseline & $\Delta_{\mathrm{temporal}}$ & $\Delta_{\mathrm{rnd}}$ & Excess $=\Delta_{\mathrm{temporal}}-\Delta_{\mathrm{rnd}}$ \\
\midrule
VideoMAE finetuned      & 109 & 0.954 & $-0.021$ & $-0.009$ & $-0.012$ \\
V-JEPA2          & 62 & 0.951 & $-0.017$ & $-0.022$ & $+0.005$ \\
TimeSformer L10  & 16 & 0.914 & $-0.011$ & $-0.001$ & $-0.010$ \\
\bottomrule
\end{tabular}
\caption{Binary classification performance under span erasure.}
\label{tab:binary_classification_span_erasure}
\end{table}

The results suggest that erasing neither the span of the temporal features nor that of an equivalent count of non-temporal features affects binary classification performance meaningfully. The Excess $=\Delta_{\mathrm{temporal}}-\Delta_{\mathrm{rnd}}$ is also non-significant compared to the baseline performance, and this seems to suggest that not only are the motion-related temporal SAE features not important to classify these kinematic pairs, they are also just as non-important as an equivalent number of the rest of the non-temporal features. We believe this is considerable evidence that SSv2 classification does not actually require temporal feature encoding in representation space, and that video models like VideoMAE finetuned and V-JEPA2 indeed don't specially utilize temporal features encoded in representation space for these supposedly temporally demanding tasks. This doesn't mean that SSv2 classification is not temporally demanding, as we discuss in Section~\ref{app:interventionconclusion}.

\subsection{Flipper experiments}

Now a question naturally presents itself, if temporal features aren't important for temporally demanding tasks, are there cases where temporal features matter at all? We show here examples of temporal features that flip the classification result between opposing pairs. We clamp each feature individually and set that feature's activation to the constant $s = \pm 150$ on all tokens, then decode and re-add the reconstruction error before reading the classifier. The magnitude $150$ is large relative to the features' typical activations. Each feature is screened on 10 SSv2 class pairs (Table~\ref{tab:pairs}) hand-picked using a motion-based heuristic from the $\tau_k$ represented in the original stimulus set, and 10 random class pairs (Table~\ref{tab:pairs_rnd}).

We count a video as \emph{flipped} when, under the clamp, the model's top-1 prediction over all 174 classes becomes the paired class. Each pair uses 12 videos per class (24 per pair). The \emph{flip rate} of a feature on a pair is the flipped fraction of the pair's videos, taken as the largest rate over the two clamp signs ($\pm 150$) and the two choices of which class of the pair is the target, and a feature is a \emph{flipper} if its flip rate reaches $0.5$ on some pair. We also report a weaker rank-based rate, the fraction of videos where the paired class merely outranks the video's own class, even if a third class is predicted.

Tables~\ref{tab:flippers} and~\ref{tab:flippers_rnd} give a list of all features considered flippers, showing that some temporal features do indeed have cases where they could be useful for some downstream task, or to intervene on for some desired effect. That being said, Table~\ref{tab:flipper_counts} also shows that by sampling an equal number of low static features from the SAE dictionary we can also find low static features that are flippers, and once again different models behave differently, matching what we observed in all the other intervention experiments. Thus, the existence of temporal flippers doesn't contradict the rest of the intervention results.

\begin{table}[h]
\centering
\small
\begin{tabular}{llllcc}
\toprule
Model & Feature & Tag & Pair & Flip rate (top-1) & Rank-based rate \\
\midrule
\multirow{4}{*}{VideoMAE finetuned}
 & feat00250 & $v_x$ & spin continue/stop & 0.67 & 1.00 \\
 & feat00214 & $v_x$ & camera left/right & 0.62 & 0.71 \\
 & feat02818 & speed & fall fast/slow & 0.50 & 0.62 \\
 & feat05077 & $v_x$ & fall fast/slow & 0.50 & 0.67 \\
\midrule
\multirow{3}{*}{V-JEPA2}
 & feat04854 & $|a|$ & spin continue/stop & 0.67 & 0.83 \\
 & feat03101 & $|a|$ & spin continue/stop & 0.58 & 0.92 \\
 & feat04594 & $|a|$ & spin continue/stop & 0.58 & 0.83 \\
\midrule
\multirow{3}{*}{TimeSformer L10}
 & feat02383 & $|a|$ & camera left/right & 0.50 & 1.00 \\
 & feat03534 & $|a|$ & camera left/right & 0.50 & 1.00 \\
 & feat05791 & $|a|$ & moving up/down & 0.50 & 1.00 \\
\bottomrule
\end{tabular}
\caption{All (feature, pair) entries with top-1 flip rate $\ge 0.5$ (flippers), per model, for the hand-picked opposing pairs.}
\label{tab:flippers}
\end{table}

\begin{table}[h]
\centering
\small
\begin{tabular}{llllcc}
\toprule
Model & Feature & Tag & Pair & Flip rate (top-1) & Rank-based rate \\
\midrule
\multirow{2}{*}{VideoMAE finetuned}
 & feat05622 & speed & tilt fall/stay & 0.67 & 1.00 \\
 & feat02951 & $v_x$ & poke move/still & 0.54 & 0.83 \\
\midrule
\multirow{2}{*}{V-JEPA2}
 & feat03759 & $|a|$ & roll down/up & 0.54 & 0.75 \\
 & feat04022 & $|a|$ & poke move/still & 0.50 & 0.83 \\
\midrule
TimeSformer L10 & - & - & - & - & -\\
\bottomrule
\end{tabular}
\caption{All (feature, pair) entries with top-1 flip rate $\ge 0.5$, per model, on the randomly drawn opposing pairs. TimeSformer's identified features produce no flipper (maximum $0.46$).}
\label{tab:flippers_rnd}
\end{table}

\begin{table}[h]
\centering
\small
\begin{tabular}{lcccc}
\toprule
& \multicolumn{2}{c}{Hand-picked pairs} & \multicolumn{2}{c}{Random pairs} \\
\cmidrule(lr){2-3}\cmidrule(lr){4-5}
Model & Identified & Low static & Identified & Low static \\
\midrule
VideoMAE finetuned & 4 & 0 & 2 & 0 \\
V-JEPA2 & 3 & 3 & 2 & 1 \\
TimeSformer & 3 & 1 & 0 & 3 \\
\bottomrule
\end{tabular}
\caption{Count of flippers (feature, pair entries with top-1 flip rate $\ge 0.5$) that are temporal and low static, on the hand-picked pairs and the randomly drawn pairs. Note that no flippers for both the hand-picked pairs and random pairs, temporal nor static, are for negative control pairs. This may suggest that features encoding negative control pair related concepts do not significantly encode concepts related to the other pairs tested.}
\label{tab:flipper_counts}
\end{table}

\subsection{What does all this mean?}
\label{app:interventionconclusion}

Overall, the results across all three experiments are consistent. Interventions on temporal SAE features can produce a greater effect on performance above a non-temporal baseline only for some models but not others. One way to interpret this is that SSv2 may not be a very temporally demanding dataset, and that it can be mostly solved using static features from single frames and the like. We believe this is unlikely, because SSv2 is well known for its temporal difficulty. What is more likely is that because any model needs a classification head or MLP, that static features interact in ways that produce temporal concepts in the classification head, so that even if the temporal features in representation space are not used, static features themselves composed over multiple frames like ``ball at position $x$ in frame $t$'' and ``ball at position $x-\delta$ in frame $t+1$'' together produce the temporal ``understanding'' the model needs to perform well.

That being said, this hypothesis itself raises two questions that we encourage future work on. First, is this classification head based mechanism really how many video models perform well on temporally demanding tasks? Second, why is it that video model pretraining fails to produce impactful temporal features? Is it because of the training objective or the pretraining dataset, or just the nature of real world data and interaction in general?

We acknowledge there are also less exciting explanations to these intervention results like SAE failure and NFP test false negatives (L2). Some of these could be resolved with more compute.

\subsection{Lists of class pairs used}
\label{app:classpairs}

\begin{table}[H]
\centering
\small
\begin{tabular}{llll}
\toprule
Key & Axis & Positive class & Negative class \\
\midrule
push\_lr & $v_x$ & Pushing [sth] from left to right & Pushing [sth] from right to left \\
pull\_lr & $v_x$ & Pulling [sth] from left to right & Pulling [sth] from right to left \\
cam\_lr & $v_x$ & Turning the camera left & Turning the camera right \\
move\_ud & $v_y$ & Moving [sth] up & Moving [sth] down \\
cam\_ud & $v_y$ & Turning the camera upwards & Turning the camera downwards \\
fall\_speed & speed & [Sth] falling like a rock & [Sth] falling like a feather or paper \\
spin\_stop & $|a|$ & Spinning [sth], continues & Spinning [sth], quickly stops \\
obj\_cam & depth & Moving [sth] towards the camera & Moving [sth] away from the camera \\
obj\_obj & depth & Moving [sth] closer to [sth] & Moving [sth] away from [sth] \\
cam\_appr & depth & Approaching [sth] with your camera & Moving away from [sth] with your camera \\
\bottomrule
\end{tabular}
\caption{The 10 screen pairs. ``Turning the camera \ldots'' classes abbreviate ``\ldots while filming [something]''. Depth pairs serve as internal negative controls to check whether temporal features flip non-motion concepts, which is possible due to feature polysemanticity.}
\label{tab:pairs}
\end{table}

\begin{table}[H]
\centering
\small
\begin{tabular}{llll}
\toprule
Key & Axis & Positive class & Negative class \\
\midrule
push\_falloff & $|a|$ & Pushing [sth] so it almost falls off & Pushing [sth] so it falls off the table \\
poke\_move & speed & Poking [sth] so it slightly moves & Poking [sth] so it almost doesn't move \\
roll\_updown & $v_x$ & Letting [sth] roll down a slanted surface & Letting [sth] roll up, so it rolls back \\
throw\_catch & $|a|$ & Throwing [sth] in the air and catching it & Throwing [sth] in the air, letting it fall \\
move\_apart & $v_x$ & Moving [sth] and [sth] closer & Moving [sth] and [sth] away \\
lift\_slide & $|a|$ & Lifting a surface, not enough to slide & Lifting a surface until [sth] slides down \\
tilt\_fall & $|a|$ & Tilting [sth] slightly so it doesn't fall & Tilting [sth] until it falls off \\
hold\_fb & control & Holding [sth] in front of [sth] & Holding [sth] behind [sth] \\
cam\_appr & control & Approaching [sth] with your camera & Moving away from [sth] with your camera \\
drop\_fb & control & Dropping [sth] in front of [sth] & Dropping [sth] behind [sth] \\
\bottomrule
\end{tabular}
\caption{The 10 randomly drawn screen pairs (draw seed 0, 91 draws). Class texts are abbreviated. Control pairs serve a similar purpose as the depth pairs in Table~\ref{tab:pairs}, as an internal control to check whether temporal features flip pairs that are non-motion related, which is possible due to feature polysemanticity. Note that the random draw coincidentally selected `cam\_appr', which also appears as a hand-picked depth control. This is a coincidence of the draw, not a deliberate repeat.}
\label{tab:pairs_rnd}
\end{table}

\section{The strongly entangled stimulus set}
\label{app:entangled}

To see why disentangling the NFP stimulus set is important, we return to the covariance formula. 
$$C_w^{(k)}(V) \;=\; \frac{1}{T}\sum_{t=1}^{T}\big(\psi_w(V,t) - \bar\psi_w(V)\big)\big(\tau_k(V,t) - \bar\tau_k(V)\big)$$
If the NFP stimulus set is entangled, meaning we have $\text{Cov}(\tau_k, \tau_{k'})\gg0$, then $C_w^{(k)}$ and $C_w^{(k')}$ behave similarly enough for most $w$ flagged as temporal for $\tau_k$ to also be flagged for $\tau_{k'}$. While this isn't technically a problem, any entangled results may give us a wrong impression and lead to incorrect insights. This raises the question of what a fair or standard entanglement would look like, because the flagging and counts will change with entanglement. The only obvious fair entanglement is to have no entanglement, which we achieve through decorrelated stimulus in Appendix~\ref{app:decorrelatestimulus}. 

To demonstrate this risk empirically, we construct a stimulus set where two of the chosen $\tau_k$ are strongly entangled. Intuitively, we create videos that show a ball moving in a fixed direction with speed $s$ at time $t$ defined by $s(t) = s_0\, r^t$. Here the frame-to-frame speed change is $|a(t)| = |r-1|\, s(t)$, exactly proportional to the speed, so speed and $|a|$ are perfectly correlated within the video. We assemble the final set by choosing how many videos of each profile type to include and solving a linear program for per-video weights as per the decorrelated stimulus procedure. The solution targets $0.9$ dataset-level mean within-video correlation of speed and $|a|$ while every other pair has near zero correlation. We then draw the starting movement direction uniformly. Table~\ref{tab:entangled_couplings} reports the couplings measured on the rendered videos.

\begin{table}[H]
\centering
\begin{tabular}{lrr}
\toprule
Pair & Mean within-video covariance & Mean within-video correlation \\
\midrule
speed--$|a|$ (target) & $+0.146$ & $+0.876$ \\
$|a|$--direction & $-0.000$ & $+0.124$ \\
speed--direction & $-0.000$ & $+0.117$ \\
$v_x$--direction & $-0.000$ & $+0.066$ \\
$v_y$--direction & $+0.001$ & $+0.060$ \\
speed--$v_x$ & $+0.001$ & $+0.014$ \\
speed--$v_y$ & $+0.003$ & $+0.009$ \\
$v_x$--$v_y$ & $-0.001$ & $+0.005$ \\
$|a|$--$v_y$ & $+0.001$ & $+0.004$ \\
$|a|$--$v_x$ & $+0.000$ & $-0.006$ \\
\bottomrule
\end{tabular}
\caption{Within-video couplings of the strongly entangled NFP set, recomputed from the 3000 rendered clips. The speed--$|a|$ pair has a  correlation of $0.88$ as rendered (target 0.9). Every other pair has covariance $\le 0.003$. Acceleration magnitude has a smaller within-video scale than speed (std $0.20$ versus $0.65$), since $|a(t)| = |v(t{+}1) - v(t)|$ is a difference of adjacent velocities.}
\label{tab:entangled_couplings}
\end{table}

The NFP test result using the strongly entangled NFP set can be found in Table~\ref{tab:entangled_conflation}. We see that when speed and acceleration have high correlation with each other, there is a significant overlap between the subsets of features sensitive to speed and to acceleration. If we look at the results without prior context, we might think that it is the nature of the model to focus encodings on speed and acceleration, and that all video models for some reason encode speed and acceleration through a joint mechanism, when in fact these numbers are just a coincidence of entanglement between speed and acceleration in the stimulus set. By decorrelating the stimulus set, we prevent such a misunderstanding from occurring.

\begin{table}[H]
\centering
\begin{tabular}{llrrrrrr}
\toprule
Model & Layer & speed & $v_x$ & $v_y$ & $|a|$ & dir & $|a|$ also speed \\
\midrule
VideoMAE finetuned & 11 & 51 & 17 & 20 & 59 & 0 & 45 (76\%) \\
VideoMAE & 11 & 32 & 0 & 0 & 43 & 0 & 26 (60\%) \\
TimeSformer & 11 & 2 & 0 & 0 & 2 & 0 & 2 (100\%) \\
V-JEPA2 & 23 & 64 & 0 & 0 & 62 & 0 & 50 (81\%) \\
DINOv2 & final & 0 & 0 & 0 & 0 & 0 & 0 \\
\bottomrule
\end{tabular}
\caption{NFP features flagged on the strongly entangled set. Every test uses the model's $\ell_1$ SAE dictionary applied at the residual stream after the listed layer's MLP; DINOv2 uses its SAE on the final-layer spatial patch tokens. Each feature is called significant by the within-video covariance $t$-test at Bonferroni $p < 0.05/D$ with non-zero covariance, where $D$ is the full dictionary size (6144, or 8192 for V-JEPA2).}
\label{tab:entangled_conflation}
\end{table}

\newpage
\section{Examples of NFP Test Videos}
\begin{figure}[h]
\centering
\includegraphics[width=0.9\linewidth]{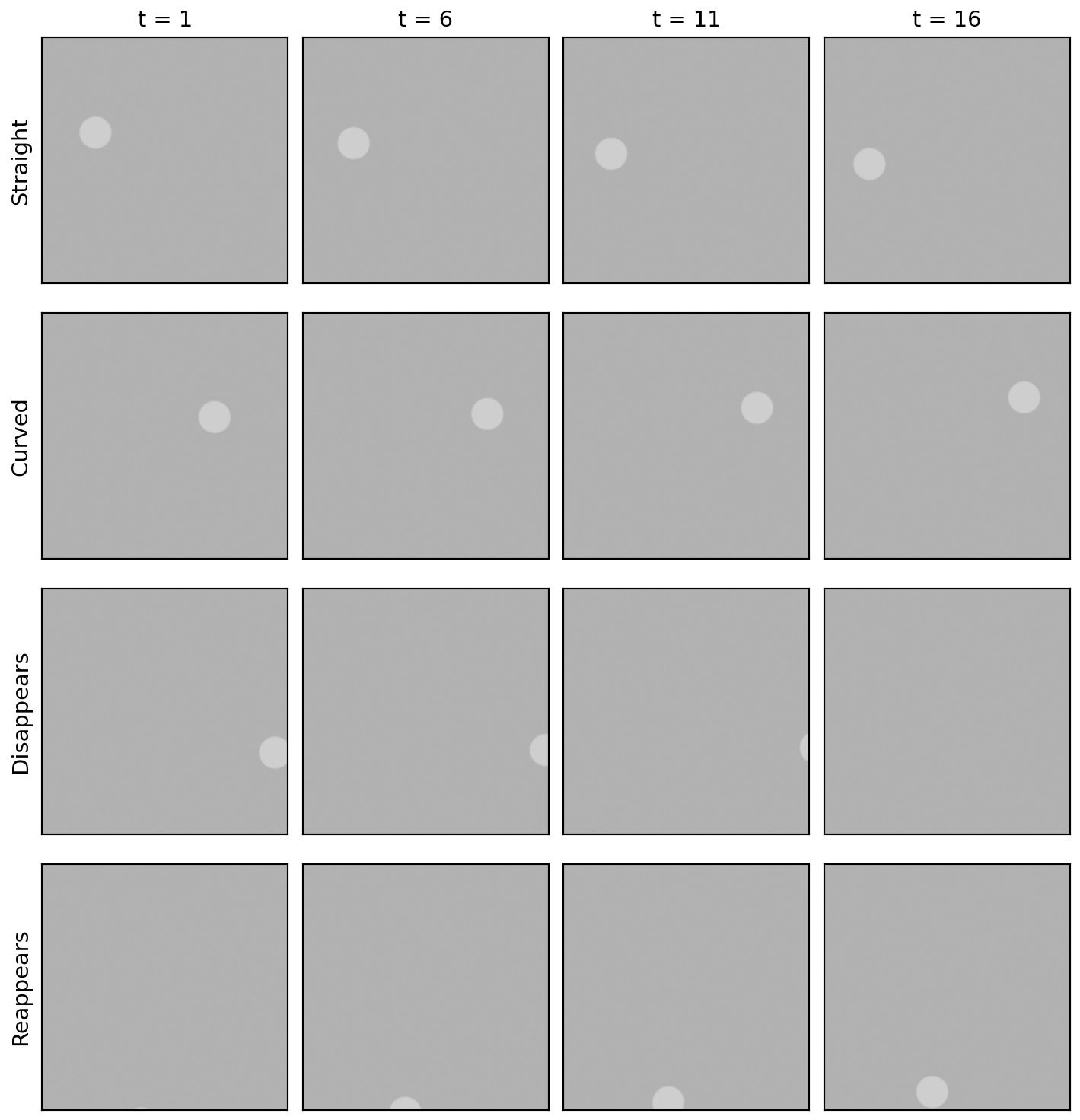}
\caption{Frame decomposition of four example NFP test videos across 4 time stamps. The NFP test videos vary in ball movement pattern and starting position but not colors or sizes of the ball or the background.}
\label{fig:video_examples}
\end{figure}
\end{document}